\documentclass[twoside,11pt]{article}
\usepackage{jmlr2e}

\usepackage[colorlinks=true, citecolor = blue, linkcolor=blue,urlcolor = blue, filecolor = blue]{hyperref}
\usepackage{amsmath,amsfonts,amssymb,xspace,bm, verbatim,dsfont,mathtools}
\usepackage{graphicx}
\usepackage{url,algorithm,algorithmic}
\usepackage{color,natbib,natbib}

\jmlrheading{x}{20xx}{xx-xx}{xx/xx}{xx/xx}{Yudong Chen, Srinadh Bhojanapalli, Sujay Sanghavi and Rachel Ward}
\ShortHeadings{Completing Any Low-rank Matrix, Provably}{Chen, Bhojanapalli, Sanghavi and Ward}
\firstpageno{1}
  
\newcommand{\note}[1]{\marginpar{\tiny *note in TeX*}}
\newcommand{\ignore}[1]{}

\renewcommand{\phi}{\varphi}

\newcommand{\R}{\mathbb{R}}

\newcommand{\eqdef}{\stackrel{\textrm{def}}{=}}

\def\mP{\mathcal{P}}

\begin{document}

\title{Completing Any Low-rank Matrix, Provably}

\author{\name Yudong Chen \email yudong.chen@eecs.berkeley.edu\\
       \addr Department of Electrical Engineering and Computer Sciences\\
       University of California, Berkeley\\
       Berkeley, CA 94704, USA
       \AND
       \name Srinadh Bhojanapalli \email bsrinadh@utexas.edu\\
       \name Sujay Sanghavi \email sanghavi@mail.utexas.edu\\
       \addr Department of Electrical and Computer Engineering\\
       The University of Texas at Austin\\
       Austin, TX 78712, USA
       \AND
       \name Rachel Ward \email rward@math.utexas.edu\\
       \addr Department of Mathematics and ICES\\
       The University of Texas at Austin\\
       Austin, TX 78712, USA
}

\editor{}

\maketitle

\begin{abstract}%
Matrix completion, i.e., the exact and provable recovery of a low-rank matrix from a small subset of its elements, is currently only known to be possible if the matrix satisfies a restrictive structural constraint---known as {\em incoherence}---on its row and column spaces. In these cases, the subset of elements is sampled uniformly at random.

In this paper, we show that {\em any} rank-$ r $ $ n$-by-$ n $ matrix can be exactly recovered from as few as $O(nr \log^2 n)$ randomly chosen elements, provided this random choice is made according to a {\em specific biased distribution}: the probability of any element being sampled should be proportional to the sum of the leverage scores of the corresponding row, and column. Perhaps equally important, we show that this specific form of sampling is nearly necessary, in a natural precise sense; this implies that other perhaps more intuitive sampling schemes fail. 

We further establish three ways to use the above result for the setting when leverage scores are not known \textit{a priori}: (a) a sampling strategy for the case when only one of the row or column spaces are incoherent, (b) a two-phase sampling procedure for general matrices that first samples to estimate leverage scores followed by sampling for exact recovery, and (c) an analysis showing the advantages of weighted nuclear/trace-norm minimization over the vanilla un-weighted formulation for the case of non-uniform sampling.
\end{abstract}

\newpage

\section{Introduction}

Low-rank matrix completion has been the subject of much recent study due to its application in myriad tasks:  collaborative filtering, dimensionality reduction, clustering, non negative matrix factorization and localization in sensor networks. Clearly, the problem is ill-posed in general; correspondingly, analytical work on the subject has focused on the joint development of algorithms, and sufficient conditions under which such algorithms are able to recover the matrix. 

While they differ in scaling/constant factors, all existing sufficient conditions~\citep{candes2009exact,candes2010power, recht2009simpler, keshavan2010matrix, gross2011recovering, jain2012low,negahban2012restricted}---with a couple of exceptions we describe in Section~\ref{sec:related}---require that  {\em (a)} the subset of observed elements should be uniformly randomly chosen, independent of the values of the matrix elements, and {\em (b)} the low-rank matrix  be ``incoherent" or ``not spiky''---i.e., its row and column spaces should be diffuse, having low inner products with the standard basis vectors.  Under these conditions, the matrix has been shown to be provably recoverable---via methods based on convex optimization~\citep{candes2009exact}, alternating minimization~\citep{jain2012low}, iterative thresholding~\citep{cai2010singular}, etc.---using as few as~$\Theta(nr\log n)$ observed elements for an $n\times n$ matrix of rank~$r$.

Actually, the incoherence assumption {\em is required because of the uniform sampling:} coherent matrices are those which have most of their mass in a relatively small number of elements.  By sampling entries uniformly and  independently at random, most of the mass of a coherent low-rank matrix will be missed; this could (and {\em does}) throw off all existing recovery methods.  One could imagine that if the sampling is adapted to the matrix, roughly in a way that ensures that elements with more mass are more likely to be observed, then it may be possible for {\em existing} methods to recover the full matrix.

{\bf In this paper,} we show that the incoherence requirement can be eliminated completely, provided the sampling distribution is dependent on the matrix to be recovered in the right way.  Specifically, we have the following results.

\begin{enumerate}
\item If the probability of an element being observed is proportional to the sum of the corresponding row and column leverage scores (which are local versions of the standard incoherence parameter) of the underlying matrix, then an \emph{arbitrary} rank-$r$ matrix can be exactly recovered from $\Theta(nr \log^2 n)$ observed elements with high probability, using nuclear norm minimization (Theorem~\ref{thm:random} and Corollary~\ref{cor:random}). In the case when all leverage scores are uniformly bounded from above, our results reduce to existing guarantees for incoherent matrices using uniform sampling.
Our sample complexity bound $\Theta(nr \log^2 n)$ is optimal up to a single factor of $\log^2 n$, since the degrees of freedom in an $n \times n$ matrix of rank $r$ is $2nr$. Moreover, we show that to complete a coherent matrix, it is \emph{necessary} (in certain precise sense) to sample according to the leverage scores as above (Theorem~\ref{thm:optimal}).

\item  For a matrix whose column space is incoherent and row space is arbitrarily coherent, our results immediately lead to a provably correct sampling scheme which \emph{requires no prior knowledge of the leverage scores of the underlying matrix} and has near optimal sampling complexity (Corollary~\ref{cor:row_coherent}).

\item  We provide numerical evidence that a two-phase adaptive sampling strategy, which assumes no prior knowledge about the leverage scores of the underlying matrix, can perform on par with the optimal sampling strategy in completing coherent matrices, and significantly outperforms uniform sampling (Section~\ref{sec:algo}).  Specifically, we consider a two-phase sampling strategy whereby given a fixed budget of $m$ samples, we first draw a fixed proportion of samples uniformly at random, and then draw the remaining samples according to the leverage scores of the resulting sampled matrix. 

\item Using our theoretical results, we are able to quantify the benefit of \emph{weighted} nuclear norm minimization 
over standard (unweighted) nuclear norm minimization, and provide a strategy for choosing the weights in such problems given non-uniformly distributed samples so as to reduce the sampling complexity of weighted nuclear norm
minimization to that of the unweighted formulation (Theorem~\ref{cor:general_weighted}).  Our results give the first exact recovery guarantee for weighted nuclear norm minimization, thus providing theoretical justification for its good empirical performance observed in~\cite{salakhutdinov2010collaborative,foygel2011learning,negahban2012restricted}. 

\end{enumerate}

Our theoretical results are achieved by a new analysis based on concentration bounds involving
the \emph{weighted} $\ell_{\infty,2}$ matrix norm, defined as the maximum
of the appropriately weighted row and column norms of the matrix. This differs from previous
approaches that use $\ell_{\infty}$ or unweighted $ \ell_{\infty,2} $ norm bounds~\citep{gross2011recovering,recht2009simpler,chen2013incoherence}. In some sense, using the weighted $\ell_{\infty,2}$-type bounds is natural for the analysis of low-rank matrices, because the rank is a property of the rows and columns of the matrix rather than its individual elements, and the weighted norm captures the relative importance of the rows/columns. Therefore, our bounds on the $ \ell_{\infty,2} $ norm might be of independent interest, and we expect the techniques to be relevant more generally, beyond the specific settings and algorithms considered here.

\section{Related Work}\label{sec:related}

There is now a vast body of literature on matrix completion, and an even bigger body of literature on matrix approximations; we restrict our literature review here to papers that are most directly related.

\textit{Exact Completion, Incoherent Matrices, Random Samples:} The first algorithm and theoretical guarantees for exact  low-rank matrix completion appeared in~\cite{candes2009exact}; there it was shown that nuclear norm minimization works when the low-rank matrix is incoherent, and the sampling is uniform random and independent of the matrix. Subsequent works have refined provable completion results for incoherent matrices under the  uniform random sampling model, both via nuclear norm minimization~\citep{candes2010power,recht2009simpler, gross2011recovering,chen2013incoherence}, and other methods like SVD followed by local descent~\citep{keshavan2010matrix} and alternating minimization~\citep{jain2012low}, etc. The setting with sparse errors and additive noise is also considered~\citep{candes2010matrix, chandrasekaran2011rank,chen2011low,candes2009robustPCA,negahban2012restricted}.

\textit{Matrix approximations via sub-sampling:}  Weighted sampling methods have been widely considered in the related context of matrix \emph{sparsification}, where one aims to approximate a given large dense matrix with a sparse matrix.  The strategy of element-wise matrix sparsification was introduced in~\cite{achlioptas2007fast}. 
 They propose and provide bounds for the \emph{$\ell_2$ element-wise sampling} model, where elements of the matrix are sampled with probability proportional to their squared magnitude.  These bounds were later refined in~\cite{drineas2011note}.  Alternatively,~\cite{arora2006fast} propose the \emph{$\ell_1$ element-wise sampling} model, where elements are sampled with probabilities proportional to their magnitude.  This model was further investigated in~\cite{achlioptas2013matrix} and argued to be almost always preferable to $\ell_2$ sampling.  
 
Closely related to the matrix sparsification problem is the matrix column selection problem, where one aims to find the ``best" $k$ column subset of a matrix to use as an approximation.  State-of-the-art algorithms for column subset selection~\citep{boutsidis2009columnselection, mahoney2011} involve randomized sampling strategies whereby columns are selected proportionally to their \emph{statistical leverage scores}---the squared Euclidean norms of projections of the canonical unit vectors on the column subspaces.  The statistical leverage scores of a matrix can be approximated efficiently, faster than the time needed to compute an SVD~\citep{drineas2012approxLeverage}.   Statistical leverage scores are also used extensively in statistical regression analysis for outlier detection~\citep{chatterjee1986influential}.  More recently, statistical leverage scores were used in the context of graph sparsification under the name of graph resistance~\citep{spielman2011graphsparsify}.  The sampling distribution we use for the matrix completion guarantees of this paper is based on statistical leverage scores.  As shown both theoretically (Theorem~\ref{thm:optimal}) and empirically (Section~\ref{sec:sims}), sampling as such outperforms both $\ell_1$ and $\ell_2$ element-wise sampling, at least in the context of matrix completion.

\textit{Weighted sampling in compressed sensing:} This paper is similar in spirit to recent work in compressed sensing which shows that sparse recovery guarantees traditionally requiring mutual incoherence can be extended to systems which are only \emph{weakly} incoherent, without any loss of approximation power, provided measurements from the sensing basis are subsampled according to their coherence with the sparsity basis. This notion of \emph{local coherence sampling} seems to have originated in~\cite{rauhut2012sparse} in the context of sparse orthogonal polynomial expansions, and has found applications in uncertainty quantification~\citep{yang2013reweighted}, interpolation with spherical harmonics~\citep{burq2012weighted}, and MRI compressive imaging~\citep{krahmer2012beyond}.

Finally, closely related to our paper is the recent work in~\cite{ks13}, which considers matrix completion where only the row space is allowed to be coherent. The proposed adaptive sampling algorithm selects columns to observe in their entirety and requires a total of $ O(r^2n\log r) $ observed elements, which is quadratic in $r$.

\subsection{Organization}
We present our main results for coherent matrix completion in Section~\ref{sec:main_coherent}. In Section~\ref{sec:algo} we propose a two-phase algorithm that requires no prior knowledge about the underlying matrix's leverage scores. In Section~\ref{sec:weight_trace} we provide guarantees for weighted nuclear norm minimization. We provide the proofs of the main theorems in the appendix.

\section{Main Results}\label{sec:main_coherent}

The results in this paper hold for what is arguably the most popular approach to matrix completion: nuclear norm minimization. If the true matrix is $M$ with its $ (i,j) $-th element denoted by $M_{ij}$, and the set of observed elements is $\Omega$, this method guesses as the completion the optimum of the convex program:
\begin{equation}\label{eq:method}
\begin{aligned}
\min_X & \quad \quad && \|X\|_* \\
& \quad\text{s.t.} && X_{ij} ~ = ~ M_{ij} ~\text{ for $(i,j)\in\Omega.$} 
\end{aligned}
\end{equation}
where the nuclear norm $\|\cdot\|_*$ of a matrix is the sum of its singular values.\footnote{This becomes the trace norm for positive-definite matrices. It is now well-recognized to be a convex surrogate for the rank function~\citep{fazel2002matrix}.}  Throughout, we use the standard notation $f(n) = \Theta(g(n))$ to mean that $c g(n) \leq f(n) \leq C g(n)$ for some positive universal constants $c, C$.

We focus on the setting where matrix elements are revealed according an underlying probability distribution.  To introduce the distribution of interest, we first need a definition.  

\begin{definition}[Leverage Scores]
For an $n_1\times n_2$ real-valued matrix $M$ of rank $r$ whose rank-$ r $ SVD is given by $U\Sigma V^{\top}$, its (normalized) {leverage scores}\footnote{In the matrix sparsification literature~\citep{drineas2012approxLeverage,boutsidis2009columnselection} and beyond, the leverage scores of $M$ often refer to the \emph{un-normalized} quantities $\left\Vert U^{\top}e_{i}\right\Vert^2$ and $\left\Vert V^{\top}e_{j}\right\Vert^2$.}---$\mu_i(M)$ for any row $i$, and $\nu_j(M)$ for any column $j$---are defined as
\begin{equation}
\begin{split}
\mu_{i}(M): & =\frac{n_1}{r}\left\Vert U^{\top}e_{i}\right\Vert _{2}^{2},\quad i=1,2,\ldots,n_1,\\
\nu_{j}(M): & =\frac{n_2}{r}\left\Vert V^{\top}e_{j}\right\Vert _{2}^{2},\quad j=1,2,\ldots,n_2,
\end{split}
\label{eq:local_incoherence}
\end{equation}
where $ e_i $ denotes the $ i $-th standard basis with appropriate dimension.
\end{definition}
Note that the leverage scores are non-negative, and are functions of the column and row spaces of the matrix $ M $. Since $U$ and $V$ have orthonormal columns, we always have relationship $\sum_{i}\mu_{i}(M)r/n_{1}=\sum_{j}\nu_{j}(M)r/n_{2}=r.$ The standard \emph{incoherence parameter} $ \mu_0 $ of $ M $ used in the previous literature corresponds to a global upper bound on the leverage scores: $$ \mu_0 \ge \max_{i,j} \{\mu_i(M), \nu_j(M)\}. $$ Therefore, the leverage scores can be considered as the localized versions of the standard incoherence parameter.

We are  ready to state our main result, the theorem below. 

\begin{theorem}
\label{thm:random} Let $M = (M_{ij})$ be an $n_1 \times n_2$ matrix of rank $ r $, and suppose that its elements $M_{ij}$ are observed only over a subset of elements $\Omega \subset [n_1] \times [n_2]$.  There is a universal constant $c_{0}>0$ such that, if each element $(i,j)$ is independently observed with probability~$p_{ij}$, and $p_{ij}$ satisfies
\begin{align}
p_{ij} ~ & \ge ~ \min \left \{ ~ c_{0}\frac{\left(\mu_{i}(M)+\nu_{j}(M)\right)r\log^{2}(n_{1}+n_{2})}{\min\{n_1,n_2\}}, ~~ 1 ~ \right \}, \label{eq:coherent_dist}\\
p_{ij} ~ & \ge ~\frac{1}{\min\{n_1,n_2\}^{10}},\nonumber 
\end{align}
then $M$ is the unique optimal solution to the nuclear norm minimization problem~\eqref{eq:method} with
probability at least $1-5(n_{1}+n_{2})^{-10}$.   
\end{theorem}
We will refer to the sampling strategy \eqref{eq:coherent_dist} as \emph{leveraged sampling}.  
Note that the expected number of observed elements is $\sum_{i,j} p_{ij}$, and this satisfies
\begin{align*}
\sum_{i,j}p_{ij}
&\ge\max\left\{ c_{0}\frac{r\log^{2}(n_{1}+n_{2})}{\min\{n_1,n_2\}}\sum_{i,j}\left(\mu_{i}(M)+\nu_{j}(M)\right),\sum_{i,j}\frac{1}{\min\{n_1,n_2\}^{10}}\right\}\\
&=2c_{0}\max\left\{ n_{1},n_{2}\right\} r\log^{2}(n_{1}+n_{2}),
\end{align*}
which is independent of the leverage scores, or indeed any other property of the matrix.  Hoeffding's inequality implies that the actual
number of observed elements sharply concentrates around its expectation, leading to the following corollary: 
\begin{corollary}
\label{cor:random} Let $M = (M_{ij})$ be an $n_1 \times n_2$ matrix of rank $ r $.  Draw a subset $ \Omega $ of its elements by leveraged sampling according to the procedure described in Theorem~\ref{thm:random}.  
There is a universal constant $c_0>0$ such that the following holds with probability at least $1-10(n_{1}+n_{2})^{-10}$: 
the number $m$ of revealed elements is bounded by
$$
\vert \Omega \vert \leq 3c_{0}\max\left\{ n_{1},n_{2}\right\} r\log^{2}(n_{1}+n_{2})
$$
and $M$ is the unique optimal solution to the nuclear norm minimization program~\eqref{eq:method}. 
\end{corollary}

We now provide comments and discussion.

{\bf (A)} Roughly speaking, the condition given in~(\ref{eq:coherent_dist})
ensures that elements in important rows/columns (indicated by large
leverage scores $\mu_{i}$ and $\nu_{j}$) of the matrix should be
observed more often. Note that Theorem~\ref{thm:random} only stipulates that
an \emph{inequality} relation hold between $p_{ij}$ and $\left\{ \mu_{i}(M),\nu_{j}(M)\right\} $. This allows for there to be some discrepancy between the sampling distribution and the leverage scores. It also has the natural interpretation that the more the sampling distribution
$\left\{ p_{ij}\right\} $ is ``aligned'' to the leverage score pattern
of the matrix, the fewer observations are needed.

 
{\bf (B)} Sampling based on leverage scores provides close to the optimal number of sampled elements required for exact recovery (when sampled with any distribution). In particular, recall that the number of degrees of freedom of an $n\times n$ matrix of rank $r$ is $2nr(1-r/2n)$, and knowing the leverage scores of the matrix reduces the degrees of freedom by at most $ 2n $. Hence, regardless how the elements
are sampled, a minimum of $\Theta(nr)$ elements is required to recover
the matrix. Theorem~\ref{thm:random}
matches this lower bound, with an additional $O(\log^{2}(n))$ factor.

{\bf (C)} Our work improves on existing results {\em even} in the case of uniform sampling and uniform incoherence. Recall that the original work of~\cite{candes2009exact}, and subsequent works~\citep{candes2010power, recht2009simpler, gross2011recovering} give recovery guarantees based on two parameters of the matrix $M\in \mathbb{R}^{n\times n}$ (assuming its SVD is $ U\Sigma V^\top $): (a) the (above-defined) {\em incoherence parameter} $\mu_0$, which is a uniform bound on the  leverage scores, 
and (b) a  {\em joint incoherence parameter} $\mu_{\text{str}}$ defined by $\|UV^\top\|_\infty = \sqrt{\frac{r\mu_{\text{str}}}{n^2}}$. With these definitions, the current state of the art states that if the  sampling probability is uniform and  satisfies 
\begin{align*}
p_{ij} \equiv p \geq c \frac{\max\{\mu_0,\mu_{\text{str}}\} r \log^2 n}{n},\quad \forall i,j,
\end{align*}
where $c$ is a constant, then $M$ will be the unique optimum of~(\ref{eq:method}) with high probability. A direct corollary of our work improves on this result, by removing the need for extra constraints on the joint incoherence; in particular, it is easy to see that our theorem implies that a uniform sampling probability of $p\geq 
c\frac{\mu_0 r \log^2 n}{n}$---that is, with no $\mu_{\text{str}}$---guarantees recovery of $M$ with high probability. Note that $\mu_{\text{str}}$ can be as high as $\mu_0 r$, for example, in the case when $ M $ is positive semi-definite; our corollary thus removes this sub-optimal dependence on the rank and on the  incoherence parameter. This improvement was recently observed in~\cite{chen2013incoherence}. 

\subsection{Knowledge-Free Completion for Row Coherent Matrices}
\label{sec:row_coherent}

Theorem~\ref{thm:random} immediately yields a useful result in scenarios where only the row space of a matrix is coherent and one has control over the sampling of the matrix. This setting is considered before by~\cite{ks13} and is of interest in applications like recommender systems, network tomography and gene expression analysis. 

Suppose the column space of $ M \in\R^{n\times n}$ is incoherent with $ \max_i \mu_i(M) \le \mu_0 $ and the row space is arbitrary (we consider square matrix for simplicity). We choose each row of $ M $ with probability $ c_0 \mu_0r\log n/n $ ($ c_0 $ is a constant), and observe all the elements of the chosen rows. We then compute the leverage scores $ \{\tilde{\nu}_j \}$ of the space spanned by these rows, and use them as estimates for $ \nu_j(M) $, the leverage scores of  $ M $. Based on these estimates, we can  perform leveraged sampling according to~\eqref{eq:coherent_dist} and then use nuclear norm minimization to recover $ M $. Note that this procedure does not require any prior knowledge about the leverage scores of $ M $. The following corollary shows that the procedure is \emph{provably correct} and exactly recovers $ M $ with high probability, using a near-optimal number of samples. 
\begin{corollary}\label{cor:row_coherent}
For some universal constants $ c_0,c_1 >0$ the following holds with probability at least $ 1-10n^{-10} $.  The above procedure computes the column leverage scores of $ M $ exactly, i.e., $ \tilde{\nu}_j = \nu_j(M),\forall j\in [n] $. If we further sample a set $ \Omega $ of elements  of $ M $ with probabilities
\[
p_{ij} = \min\left\{c_0 \frac{(\mu_0 + \tilde{\nu}_j)r\log^2n}{n}, 1\right\},\quad\forall i,j,
\]
then $ M $ is the unique optimal solution to the nuclear norm minimization program~\eqref{eq:method}. The total number of samples used by the above procedure is at most $ c_1 \mu_0 rn \log^2n $.
\end{corollary}
%
Our result improves on the $ \Theta(\mu_0^2r^2n\log r) $ sample complexity given in~\cite{ks13}, which is quadratic in $ \mu_0$ and $ r $. We also note that our sampling strategy is different from theirs: we sample entire rows of $ M $, whereas they sample entire columns.

\subsection{Necessity of Leveraged Sampling}
\label{sec:optimality}

In this subsection, we show that the leveraged sampling in~\eqref{eq:coherent_dist} is necessary for completing a coherent matrix in a certain precise sense. For simplicity, we restrict ourselves to square matrices in $ \mathbb{R}^{n\times n} $. Suppose each element $ (i,j) $ is observed independently with probability $p_{ij}$. We consider a family of sampling probabilities $\{p_{ij}\}$ with the following property.
\begin{definition}[Location Invariance]
$\left\{ p_{ij}\right\} $ is said to be location-invariant
with respect to the matrix $M$ if the following are satisfied: (1) For any two rows $i\neq i'$ that are identical, i.e., $M_{ij}=M_{i'j}$ for all $j$, we have $p_{ij}=p_{i'j}$
for all $j$; (2) For any two columns $j\neq j'$ that are identical, i.e.,  $M_{ij}=M_{ij'}$ for all $i$, we have $p_{ij}=p_{ij'}$ for all $i$. 
\end{definition}
%
In other words, $\left\{ p_{ij}\right\} $ is location-invariant with respect to
$M$ if identical rows (or columns) of $M$ have identical sampling
probabilities. We consider this assumption very mild, and it covers the leveraged sampling as well as 
many other typical sampling schemes, including:
\begin{itemize}
\item uniform sampling, where $p_{ij}\equiv p$,
\item element-wise magnitude sampling, where $p_{ij}\propto\left|M_{ij}\right|$ ($ \ell_1 $ sampling) or $p_{ij}\propto M_{ij}^2  $ ($ \ell_2 $ sampling), and
\item row/column-wise magnitude sampling, where $p_{ij}\propto f\left(\left\Vert M_{i\cdot}\right\Vert _{2},\left\Vert M_{\cdot j}\right\Vert _{2}\right)$
for some (usually coordinate-wise non-decreasing) function $f:\mathbb{R}_{+}^{2}\mapsto\left[0,1\right]$.
\end{itemize}
Given two $ n $-dimensional vectors $\vec{\mu}=\left(\mu_{1},\ldots,\mu_{n}\right)$
and $\vec{\nu}=\left(\nu_{1},\ldots,\nu_{n}\right)$, we use $\mathcal{M}_{r}\left(\vec{\mu},\vec{\nu}\right)$
to denote the set of rank-$r$ matrices whose leverage scores are
bounded by $\vec{\mu}$ and $\vec{\nu}$; that is,
\[
\mathcal{M}_{r}\left(\vec{\mu},\vec{\nu}\right):=\left\{ M\in\mathbb{R}^{n\times n}:\text{rank}(M)=r;\mu_{i}(M)\le\mu_{i},\nu_{j}(M)\le\nu_{j},\forall i,j\right\} .
\]

We have the following results.
\begin{theorem}
\label{thm:optimal}Suppose $n\ge r\ge2$. Given any $2r$ numbers
$a_{1},\ldots,a_{r}$ and $b_{1},\ldots,b_{r}$ with $\frac{r}{4}\le\sum_{k=1}^{r}\frac{1}{a_{k}},\sum_{k=1}^{r}\frac{1}{b_{k}}\le r$
and $\frac{2}{r}\le a_{k},b_{k}\le\frac{2n}{r},\forall k\in[r]$,
there exist two $n$-dimensional vectors $\vec{\mu}$ and $\vec{\nu}$ and the corresponding
set $\mathcal{M}_{r}\left(\vec{\mu},\vec{\nu}\right)$ with the following
properties:
\begin{enumerate}
\item For each $i,j\in[n]$, $\mu_{i}=a_{k}$ and $\nu_{j}=b_{k'}$ for
some $k,k'\in[r]$. That is, the values of the leverage scores
are given by $\left\{ a_{k}\right\} $ and $\left\{ b_{k\rq{}}\right\} $.
\item There exists a matrix $M^{(0)}\in\mathcal{M}_{r}\left(\vec{\mu},\vec{\nu}\right)$
for which the following holds. If $\{p_{ij}\}$ is location-invariant
w.r.t. $M^{(0)}$, and for some $(i_{0},j_{0})$,
\begin{equation}
p_{i_{0}j_{0}}\le \frac{\mu_{i_{0}}+\nu_{j_{0}}}{4n}\cdot r\log\left(\frac{2n}{(\mu_{i_{0}}\vee\nu_{j_{0}})r}\right),\text{\footnotemark}\label{eq:too_small}
\end{equation}
\footnotetext{We use the notation $a \vee b =\max\{a ,b\}.$}then with probability at least $\frac{1}{4}$, the following conclusion
holds: There are infinitely many matrices $M^{(1)}\neq M^{(0)}$ in $\mathcal{M}_{r}\left(\vec{\mu},\vec{\nu}\right)$
such that $\{p_{ij}\}$ is location-invariant w.r.t. $M^{(1)}$, and
\[
M_{ij}^{(0)}=M_{ij}^{(1)},\quad \forall(i,j)\in\Omega.
\]
\item If we replace the condition~(\ref{eq:too_small}) with 
\begin{equation}
p_{i_{0}j_{0}}\le\frac{\mu_{i_{0}}+\nu_{j_{0}}}{4n}\cdot r\log\left(\frac{n}{2}\right),\label{eq:too_small2}
\end{equation}
 then the conclusion above holds with probability at least $\frac{1}{n}$.
\end{enumerate}
\end{theorem}
In other words, if~(\ref{eq:too_small}) holds, then with probability
at least $1/4$, no method can distinguish between $M^{(0)}$ and $M^{(1)}$;
similarly, if~(\ref{eq:too_small2}) holds, then with probability
at least $1/n$ no method succeeds.
We shall compare these results with Theorem~\ref{thm:random}, which guarantees that if we use leveraged sampling, 
\[
p_{ij}\ge c_0\frac{\mu_{i}+\nu_{j}}{n}\cdot r\log n,\quad \forall i,j
\]
for some universal constant $c_0$, then for
any matrix $M^{(0)}$ in $\mathcal{M}_{r}\left(\vec{\mu},\vec{\nu}\right)$,
the nuclear norm minimization approach~\eqref{eq:method} recovers $M^{(0)}$ from its
observed elements with failure probability no more than~$\frac{1}{n}$.
Therefore, under the setting of Theorem~\ref{thm:optimal}, leveraged sampling is {\bf sufficient and necessary} for matrix completion up to one logarithmic factor for  a target failure probability~$\frac{1}{n}$ (or up to two logarithmic factors for a target failure probability~$ \frac{1}{4} $).

Admittedly, the setting covered by Theorem~\ref{thm:optimal} has several  restrictions on the sampling distributions and the values of the leverage scores. Nevertheless, we believe this result captures some essential difficulties in recovering general coherent matrices, and highlights how the sampling probabilities should relate in a specific way with the leverage score structure of the underlying object.

\def\mT{\mathcal{T}}
\def\Po{\mP_{\Omega}}
\section{A Two-Phase Sampling Procedure}\label{sec:algo}

We have seen that one can exactly recover an arbitrary $n \times n$ rank-$r$ matrix using $\Theta(nr\log^2n)$ elements if sampled in accordance with the leverage scores. In practical applications of matrix completion, even when the user is free to choose how to sample the matrix elements, she may not be privy to the leverage scores $\{\mu_{i}(M),\nu_{j}(M)\}$.  In this section we propose a two-phase sampling procedure, described below and in Algorithm~\ref{alg:1}, which assumes no a priori knowledge about the matrix leverage scores, yet is observed to be competitive with the ``oracle" leveraged sampling distribution~\eqref{eq:coherent_dist}.  

\begin{algorithm}[b]
\caption{Two-phase sampling for coherent matrix completion}
\label{alg:1}
\begin{algorithmic}
\INPUT Rank parameter $r$, sample budget $m$, and parameter $\beta\in[0,1]$ 
\STATE \textbf{Step 1:} Obtain the initial set $ \Omega $ by sampling uniformly without replacement such that $| \Omega | = \beta m$. Compute best rank-$r$ approximation to $\Po(M)$,  $\tilde{U}\tilde{\Sigma}\tilde{V}^{\top}$, and its leverage scores $\{ \tilde{\mu}_i \}$ and $\{\tilde{\nu}_j\} $.
\STATE \textbf{Step 2:} Generate set of $(1-\beta)m$ new samples $\tilde{\Omega}$ by sampling without replacement with distribution~\eqref{eq:estimatedist}. Set
$$\hat{M}=\arg \min_{X} \| X \|_* \mbox{ s.t }\mP_{\Omega \cup \tilde{\Omega}}(X)=\mP_{\Omega \cup \tilde{\Omega}}(M).$$
\OUTPUT Completed matrix $\hat{M}$. 
\end{algorithmic}
\end{algorithm}
 
Suppose we are given a total budget of $m$ samples. The first step of the algorithm is to use the first $ \beta $ fraction of the budget to estimate the leverage scores of the underlying matrix, where $\beta \in [0, 1]$. Specifically, take a set of indices $\Omega$ sampled uniformly without replacement such that $|\Omega|=\beta m$, and let $\Po(\cdot)$ be the sampling operator which maps the matrix elements not in $\Omega$ to 0. Take the rank-$r$ SVD of $\Po(M)$,  $\tilde{U}\tilde{\Sigma}\tilde{V}^{\top}$, where $\tilde{U}, \tilde{V} \in \R^{n \times r}$ and $\tilde{\Sigma} \in \R^{r \times r}$, and then use the leverage scores $\tilde{\mu}_i:=\mu_i(\tilde{U}\tilde{\Sigma}\tilde{V}^\top)$ and $\tilde{\nu}_j:=\nu_j(\tilde{U}\tilde{\Sigma}\tilde{V}^\top)$ as estimates for the column and row leverage scores of $M$. 
Now as the second step, generate the remaining $(1-\beta)m$ samples of the matrix $M$ by sampling without replacement with distribution
\begin{equation} \label{eq:estimatedist}
\tilde{p}_{ij}\propto \frac{(\tilde{\mu}_i + \tilde{\nu}_j)r\log^{2}(2n)}{n}. 
\end{equation} 
Let $\tilde{\Omega}$ denote the new set of samples. Using the combined set of samples $\mP_{\Omega \cup \tilde{\Omega}}(M)$ as constraints, run the nuclear norm minimization program~\eqref{eq:method}. Let $\hat{M}$ be the optimum of this program.


To understand the performance of the two-phase algorithm, assume that the initial set of $m_1 = \beta m $ samples $\Po(M)$ are generated uniformly at random.  If the underlying matrix $M$ is incoherent, then already the algorithm will recover $M$ if $m_1 = \Theta(n r \log^2(2n))$.  On the other hand, if $M$ is \emph{highly} coherent, having almost all energy concentrated on just a few elements, then the estimated leverage scores~\eqref{eq:estimatedist} from uniform sampling in the first step will be poor and hence the recovery algorithm suffers. Between these two extremes, there is reason to believe that the two-phase sampling procedure will provide a better estimate to the underlying matrix than if all $m$ elements were sampled uniformly. Indeed, numerical experiments suggest that the two-phase procedure can indeed significantly outperform uniform sampling for completing coherent matrices.



\subsection{Numerical Experiments}\label{sec:sims}

\begin{figure*}[ht]
\vskip 0.2in
\begin{center}
\centerline{\includegraphics[width=110mm]{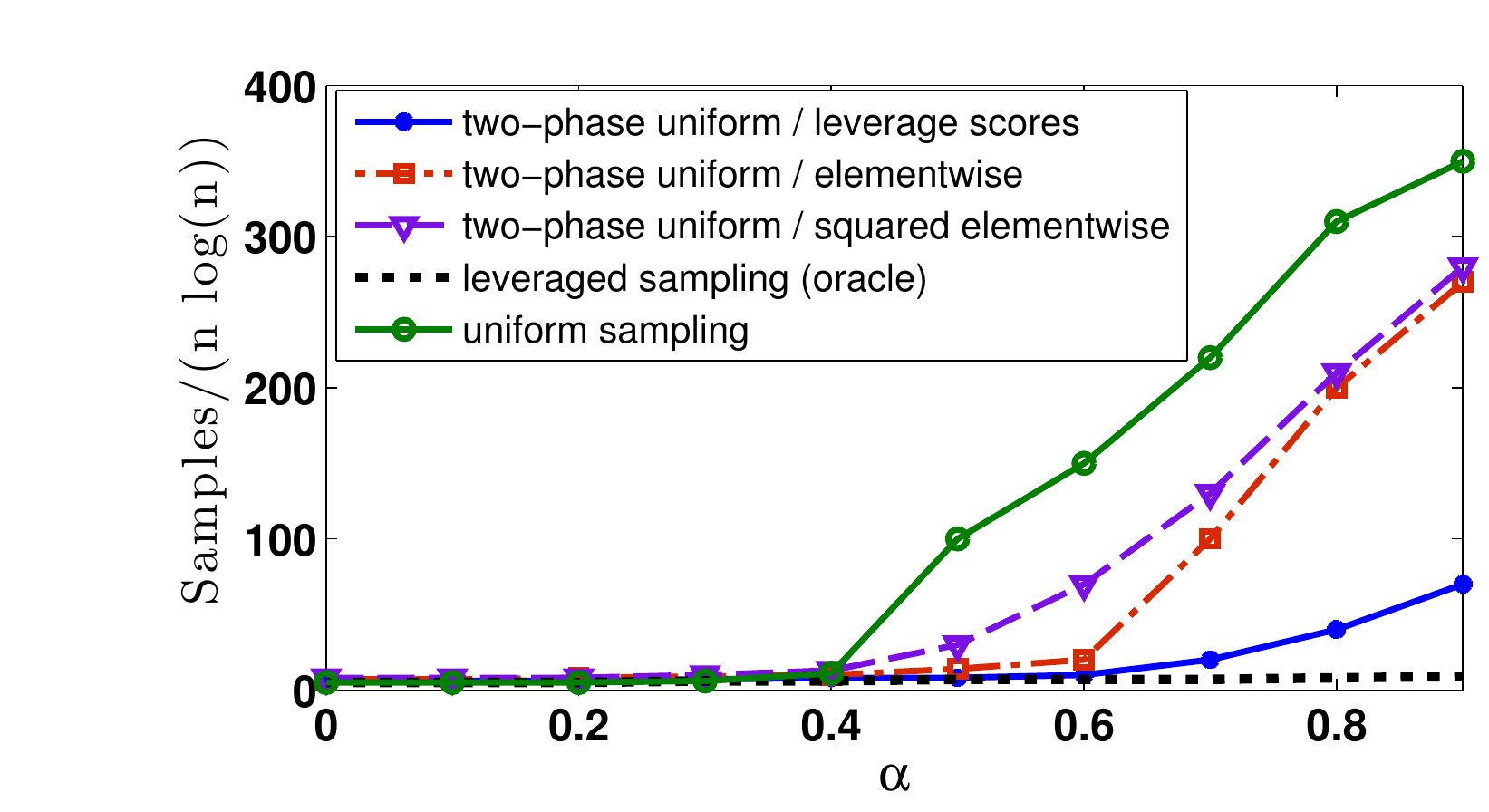}}
\caption{\small{\bf Performance of Algorithm~\ref{alg:1} for power-law matrices: } We consider rank-5 matrices of the form $M=DUV^\top D$, where elements of the matrices $U$ and $V$ are generated independently from a Gaussian distribution $\mathcal{N}(0,1)$ and $D$ is a diagonal matrix with $D_{ii}=\frac{1}{i^{\alpha}}.$  Higher values of $\alpha$ correspond to more non-uniform leverage scores and less incoherent matrices. The above simulations are run with two-phase parameter $\beta=2/3$. Leveraged sampling~\eqref{eq:coherent_dist} gives the best results of successful recovery using roughly $10n\log(n)$ samples for all values of $\alpha$ in accordance with Theorem~\ref{thm:random}. Surprisingly, sampling according to~\eqref{eq:estimatedist} with estimated leverage scores has almost the same sample complexity for $\alpha \leq 0.7$. Uniform sampling and sampling proportional to element and element squared perform well for low values of $\alpha$, but their performance degrades quickly for $\alpha >0.6$.}
\label{fig:alpha_vs_samples}
\end{center}
\vskip -0.2in
\end{figure*} 

\begin{figure*}
\vskip 0.2in
  \centering
  \begin{tabular}[ht]{ccc}
    \includegraphics[width=.45\textwidth]{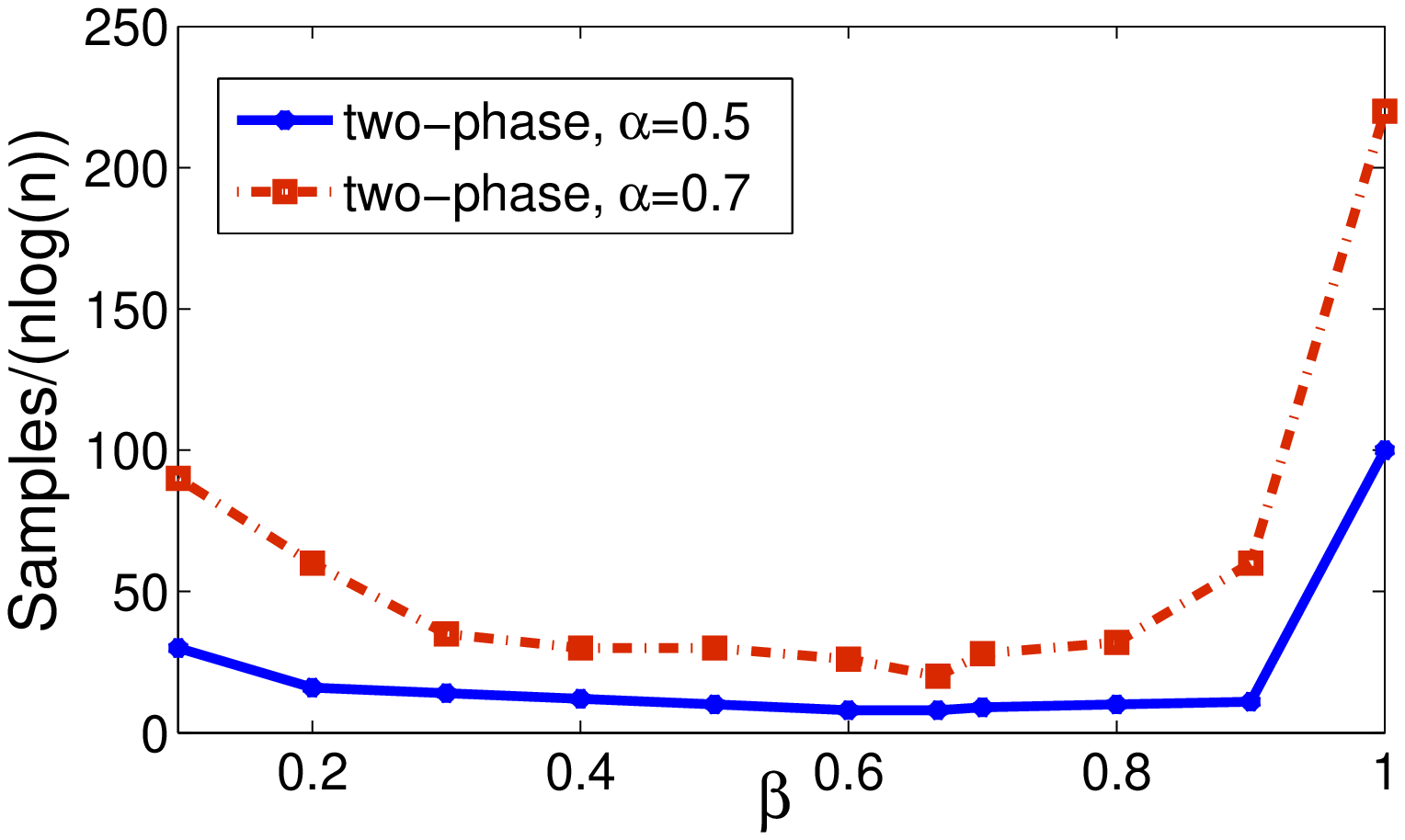}\hspace*{-0pt}&\includegraphics[width=.45\textwidth]{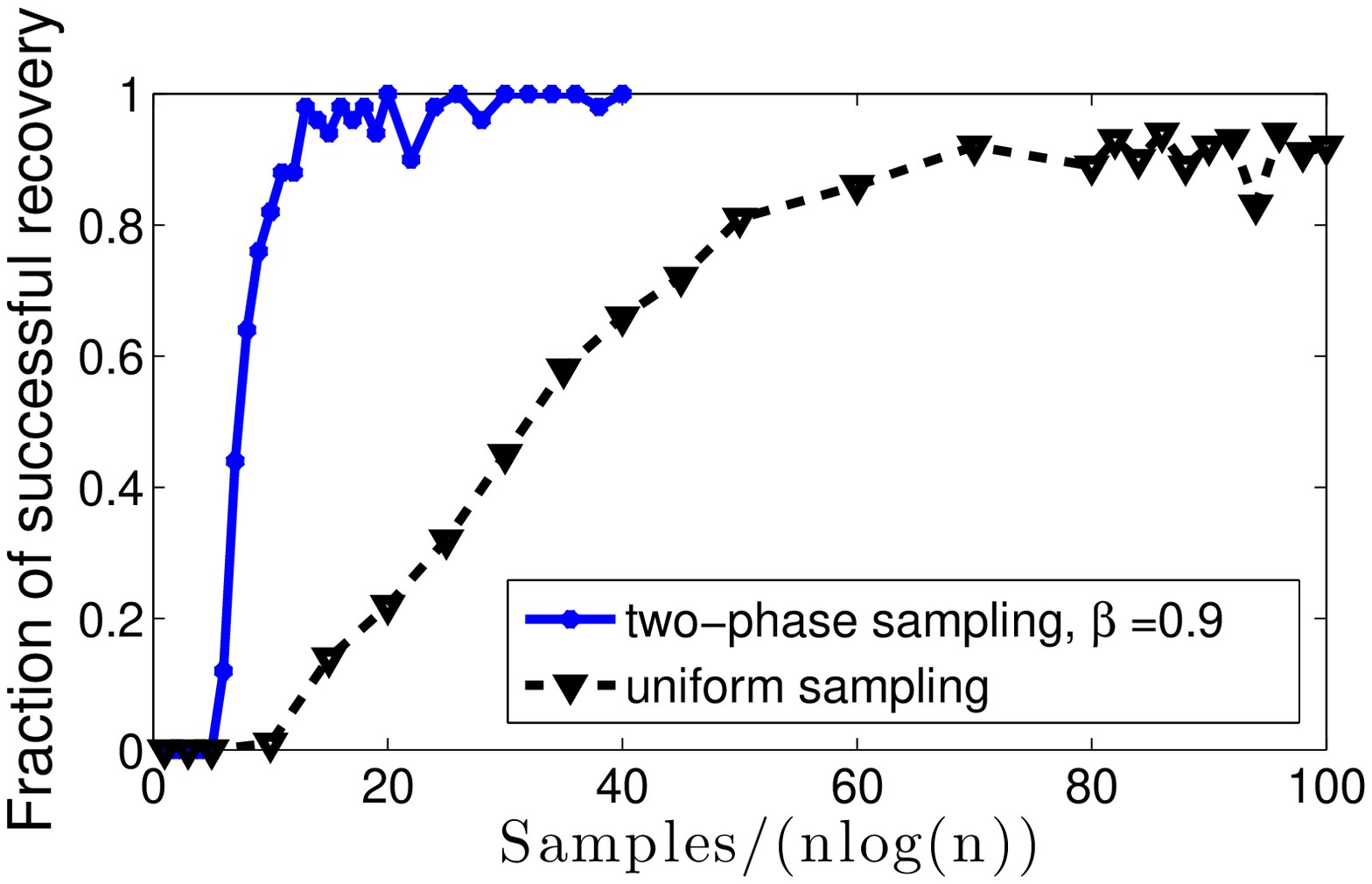}\hspace*{-15pt}\\
{\bf (a)}&{\bf (b)}
  \end{tabular}
  \caption{\small We consider power-law matrices with parameter $\alpha=0.5$ and $\alpha=0.7$.  {\bf(a):} This plot shows that Algorithm~\ref{alg:1} successfully recovers coherent low-rank matrices with fewest samples $(\approx 10n\log(n))$ when the proportion of initial samples drawn from the uniform distribution is in the range $\beta \in [0.5, 0.8].$  In particular, the sampling complexity is significantly lower than that for uniform sampling ($\beta = 1$). Note the x-axis starts at $ 0.1 $.  {\bf (b):} Even by drawing $90\%$ of the samples uniformly and using the estimated leverage scores to sample the remaining $10\%$ samples, one observes a marked improvement in the rate of recovery.}
  \label{fig:beta_vs_samples}
\vskip -0.2in
\end{figure*} 

\begin{figure*}
\vskip 0.2in
  \centering
  \begin{tabular}[ht]{ccc}
    \includegraphics[width=.45\textwidth]{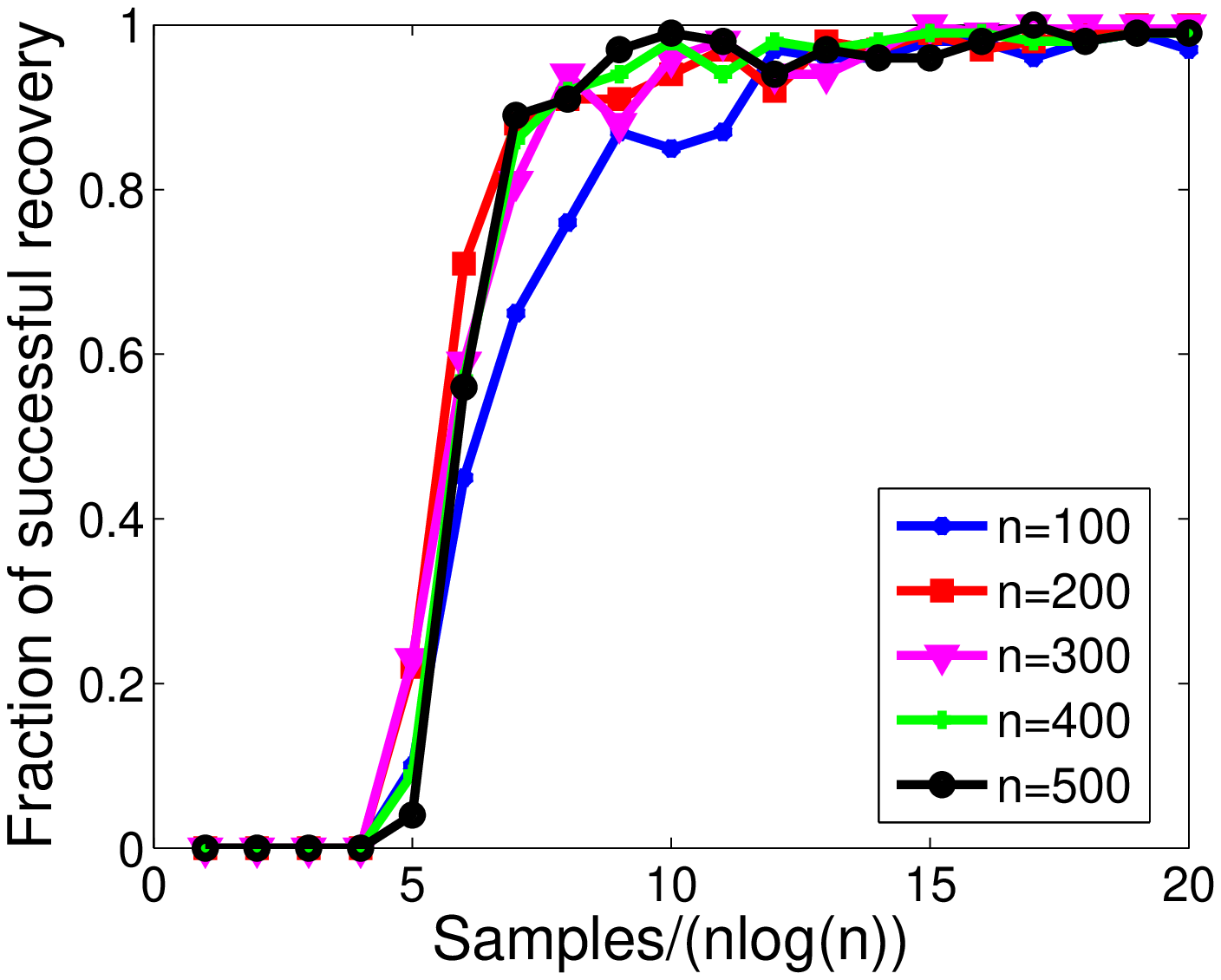}\hspace*{0pt}&\includegraphics[width=.46\textwidth]{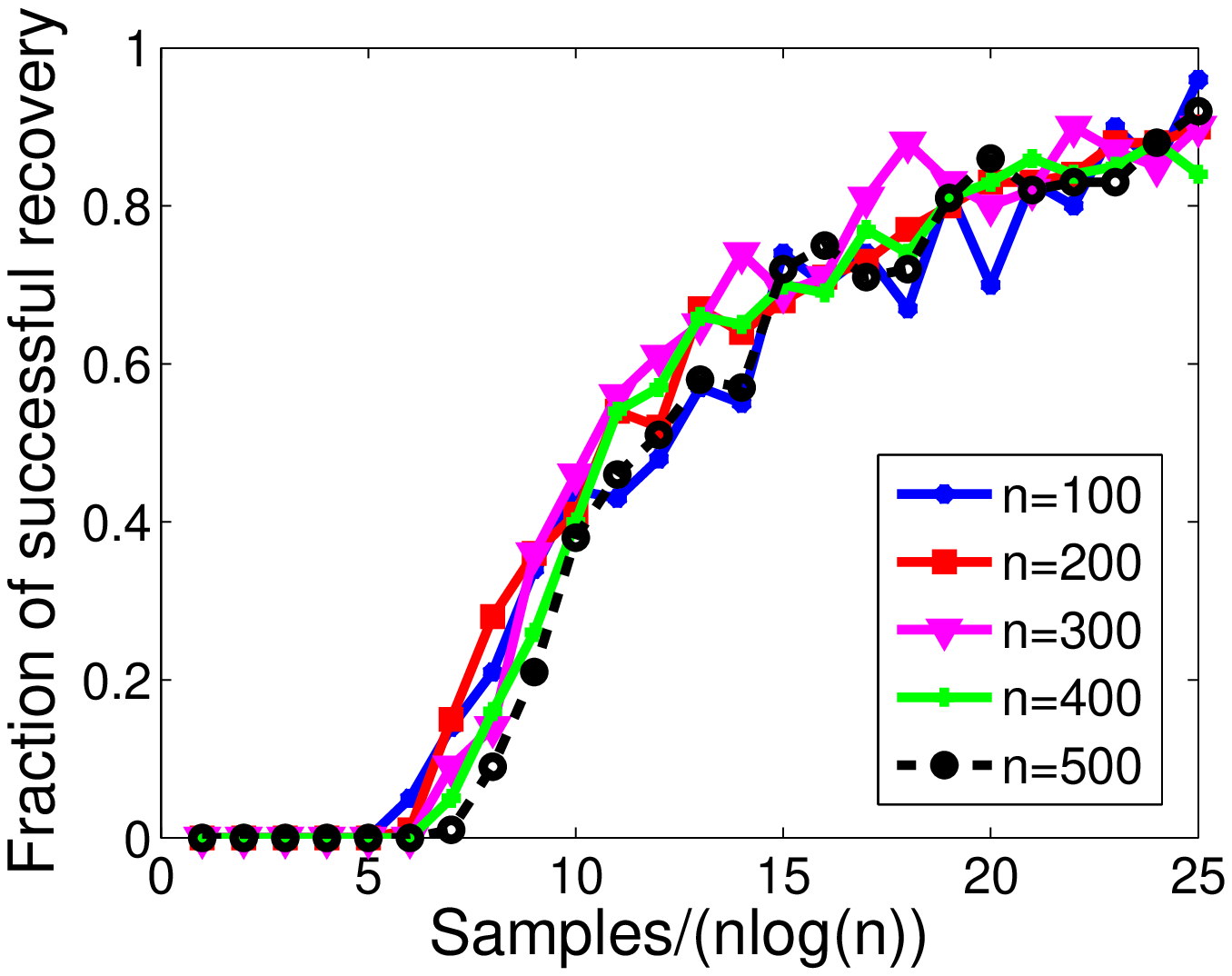}\hspace*{-5pt}\\
{\bf (a)}&{\bf (b)}
  \end{tabular}
  \caption{\small {\bf(a) \& (b)}:{\bf Scaling of sample complexity of Algorithm~\ref{alg:1} with $n$:}  We consider power-law matrices (with $\alpha=0.5$ in plot {\bf (a)} and 0.7 in plot {\bf (b)}). The plots suggest that the sample complexity of Algorithm~\ref{alg:1} scales roughly as $\Theta(n \log(n) )$.}
  \label{fig:n_vs_success}
\vskip -0.2in
\end{figure*} 

\begin{figure*}
\vskip 0.2in
  \centering
  \begin{tabular}[ht]{ccc}
    \includegraphics[width=.45\textwidth]{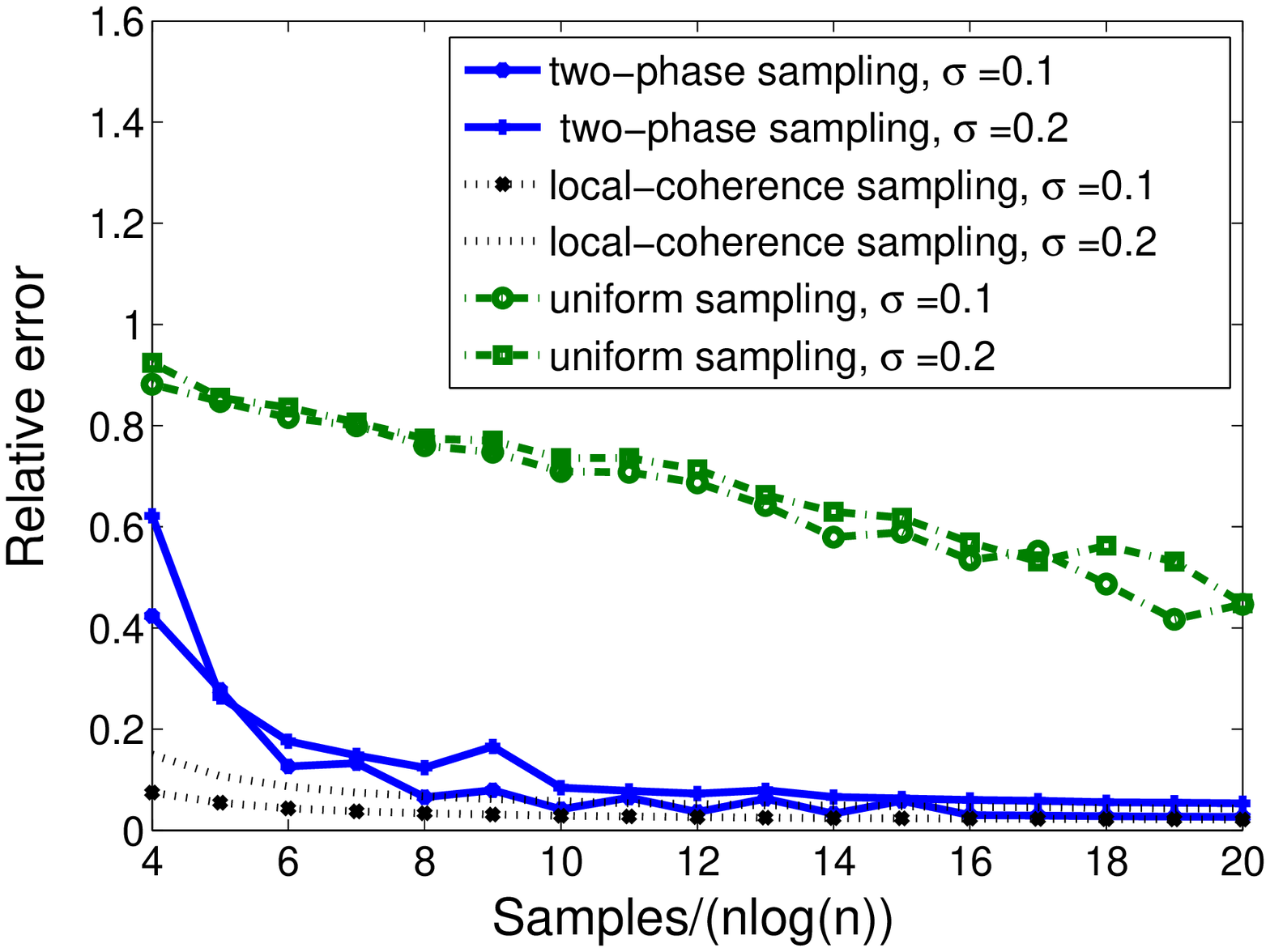}\hspace*{0pt}&\includegraphics[width=.46\textwidth]{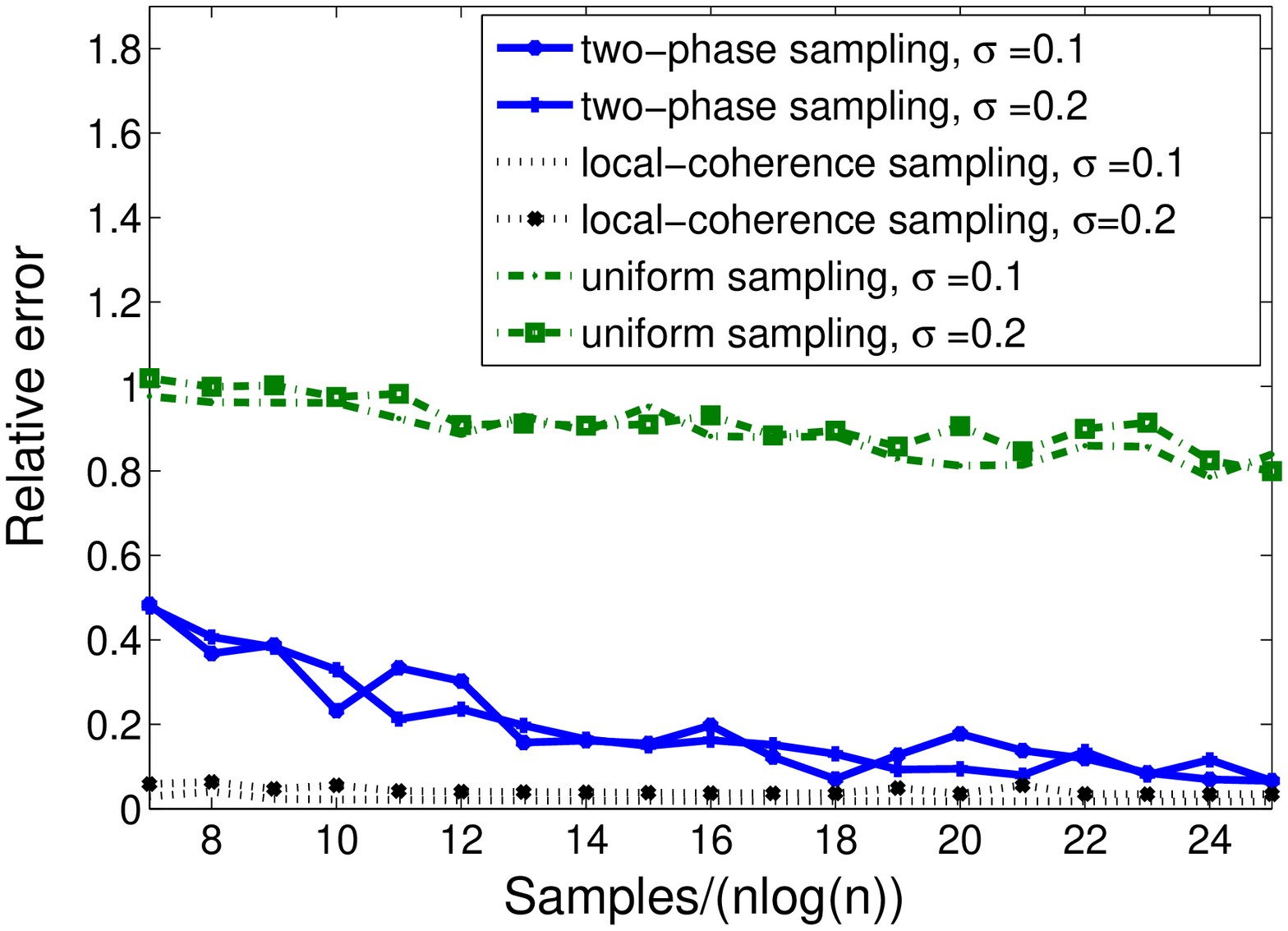}\hspace*{-15pt}\\
{\bf (a)}&{\bf (b)}
  \end{tabular}
  \caption{\small {\bf(a) \& (b)}:{\bf Performance of Algorithm~\ref{alg:1} with noisy samples:} We consider power-law matrices (with $\alpha=0.5$ in plot {\bf (a)} and $\alpha=0.7$ in plot {\bf (b)}), perturbed by a Gaussian noise matrix $Z$ with $\| Z \|_F/ \| M \|_F =\sigma$. The plots consider two different noise levels, $\sigma=0.1$ and $\sigma=0.2$. We compare two-phase sampling (Algorithm~\ref{alg:1}) with $\beta = 2/3$, sampling from the exact leverage scores, and uniform sampling. Algorithm~\ref{alg:1} has error almost as low as the leveraged sampling without requiring any a priori knowledge of the low-rank matrix, while uniform sampling suffers dramatically.}
  \label{fig:noisy_case}
\vskip -0.2in
\end{figure*}

We now study the performance of the two-phase sampling procedure outlined in Algorithm~\ref{alg:1} through numerical experiments. For this, we consider rank-$5$ matrices of size $500\times500$ of the form $M= DUV^\top D$, where the elements of the matrices $U$ and $V$ are i.i.d. Gaussian $\mathcal{N}(0,1)$ and $D$ is a diagonal matrix with power-law decay, $D_{ii}=i^{-\alpha}, 1 \leq i \leq 500.$ We refer to such constructions as \emph{power-law} matrices.  The parameter $\alpha$ adjusts the leverage scores (and hence the coherence level) of $ M $ with $\alpha=0$ being maximal incoherence $ \mu_0 = \Theta(1) $ and $\alpha=1$ corresponding to maximal coherence $\mu_0=\Theta(n)$. 

We normalize $M$ to make $\|M\|_F=1$. Figure~\ref{fig:alpha_vs_samples} plots the number of samples required for successful recovery (y-axis) for different values of $\alpha$ (x-axis) using Algorithm~\ref{alg:1} with the initial samples $\Omega$ taken  uniformly at random. Successful recovery is defined as when at least $95\%$ of trials have relative error in the Frobenius norm not exceeding 0.01. To put the results in perspective, we plot it in Figure \ref{fig:alpha_vs_samples} against the performance of pure uniform sampling, as well as other popular sampling distributions from the matrix sparsification literature~\citep{achlioptas2007fast, achlioptas2013matrix, arora2006fast, drineas2011note}, namely, in step 2 of the algorithm, sampling proportional to element ($\tilde{p}_{ij} \propto |\tilde{M}_{ij}|$) and sampling proportional to element squared ($\tilde{p}_{ij} \propto \tilde{M}_{ij}^2$), as opposed to sampling from the distribution~\eqref{eq:estimatedist}.  In all cases, the estimated matrix $\tilde{M}$ is constructed from the rank-$r$ SVD of $\Po(M)$,  $\tilde{M}=\tilde{U}\tilde{\Sigma}\tilde{V}^\top$.  Performance of nuclear norm minimization using samples generated according to the ``oracle" distribution~\eqref{eq:coherent_dist} serves as baseline for the best possible recovery, as theoretically justified by Theorem~\ref{thm:random}.  We use the Augmented Lagrangian Method (ALM) based solver in~\cite{inexactalm} to solve the convex optimization program~\eqref{eq:method}.

Figure ~\ref{fig:alpha_vs_samples} suggests that the two-phase algorithm performs comparably to the theoretically optimal leverage scores-based distribution~\eqref{eq:coherent_dist}, despite not having access to the underlying leverage scores, in the regime of mild to moderate coherence. While the element-wise sampling strategies perform comparably for low values of $\alpha$, the number of samples for successful recovery increases quickly for $\alpha >0.6$. Completion from purely uniformly sampled elements requires significantly more samples at higher values of $\alpha$.

\textit{Choosing $\beta$:} Recall that the parameter $\beta$ in Algorithm~\ref{alg:1} is the fraction uniform samples used to estimate the leverage scores. Figure~\ref{fig:beta_vs_samples}(a) plots the number of samples required for successful recovery (y-axis) as  $\beta$ (x-axis) varies from 0 to 1 for different values of $\alpha$.  $\beta=1$ reduces to purely uniform sampling, and for small values of $\beta$, the leverage scores estimated in \eqref{eq:estimatedist} will be far from the actual leverage scores.  Then, as  expected, the sample complexity goes up for $\beta$ near 0 and $\beta=1$.  We find the algorithm performs well for a wide range of $ \beta $, and setting $\beta \approx 2/3$ results in the lowest sample complexity. Surprisingly, even taking $\beta=0.9$ as opposed to pure uniform sampling $\beta = 1$ results in a significant decrease in the sample complexity; see Figure~\ref{fig:beta_vs_samples}(b) for more details.   That is, even budgeting just a small fraction of samples to be drawn from the estimated leverage scores can significantly improve the success rate in low-rank matrix recovery as long as the underlying matrix is not completely coherent. In applications like collaborative filtering, this would imply 
that incentivizing just a small fraction of users to rate a few selected movies according to the estimated leverage score distribution obtained by previous samples has the potential to greatly improve the quality of the recovered matrix of preferences.  

In Figure~\ref{fig:n_vs_success} we compare the performance of the two-phase algorithm for different values of the matrix dimension $n$, and notice for each $n$ a phase transition occurring at $\Theta(n \log(n))$ samples.  In Figure~\ref{fig:noisy_case} we consider the scenario where the samples are noisy and compare the performance of Algorithm~\ref{alg:1}  to uniform sampling and the theoretically-optimal leveraged sampling from Theorem \ref{thm:random}. Specifically we assume that the samples are generated from $M+Z$ where $Z$ is a Gaussian noise matrix. We consider two values for the noise $\sigma \eqdef \| Z \|_F/ \| M \|_F$: $\sigma=0.1$ and $\sigma=0.2$. The figures plot error in Frobenius norm $\| M-\hat{M} \|_F$ (y-axis), vs total number of samples $m$ (x-axis). These plots demonstrate the robustness of the algorithm to noise and once again show that sampling with estimated leverage scores can be as good as sampling with exact leverage scores for matrix recovery using nuclear norm minimization for $\alpha \leq 0.7$. 
    
\section{Weighted Nuclear Norm Minimization\label{sec:weight_trace}}

Theorem~\ref{thm:random} suggests that the more a set of observed elements is aligned with the leverage scores of a matrix, the better will be the performance of nuclear norm minimization.  Interestingly, Theorem~\ref{thm:random} can also  be used in a reverse way: \emph{one may adjust
the leverage scores to align with a given set of observations}. Here we demonstrate
an application of this idea in quantifying the benefit
of \emph{weighted} nuclear norm minimization for non-uniform sampling.

In many applications, the revealed elements are given to us, and distributed non-uniformly among the rows and columns.
As observed in~\cite{salakhutdinov2010collaborative}, standard unweighted nuclear norm minimization~(\ref{eq:method}) is inefficient in this setting. They propose to instead use weighted nuclear norm minimization for low-rank matrix completion:  
\begin{equation}\label{eq:weighted}
\begin{aligned}
\hat{X}=\arg\min_{X\in \mathbb{R}^{n_1\times n_2}} & \;\;  \left\Vert RXC\right\Vert _{*}\\
\textrm{s.t.} &  \;\; X_{ij}=M_{ij}, \text{ for }(i,j) \in\Omega,
\end{aligned}
\end{equation}
where $R=\textrm{diag}(R_{1},R_{2},\ldots,R_{n_{1}})\in\R^{n_1\times n_1}$ and $C=\textrm{diag}(C_{1},\ldots,C_{n_{2}})\in\R^{n_2\times n_2}$
are user-specified diagonal weight matrices with positive diagonal elements. 

We now provide a theoretical guarantee for this method, and quantify its advantage over unweighted nuclear norm minimization. Suppose $M\in\mathbb{R}^{n_1\times n_2}$ has rank $ r $ and  satisfies the {standard} incoherence condition $\max_{i,j}\left\{ \mu_{i}(M),\nu_{j}(M)\right\} \le\mu_{0}$. Let $\left\lfloor x\right\rfloor $ denote the largest integer not exceeding $x$. Under this setting, we can apply Theorem~\ref{thm:random} to establish the following:
\begin{theorem}
\label{cor:general_weighted}Without lost of generality, assume $R_{1}\le R_{2}\le\cdots\le R_{n_1}$
and $C_{1}\le C_{2}\le\cdots\le C_{n_2}$. 
There exists a universal constant $c_{0}$
such that $ M $ is the unique optimum to~(\ref{eq:weighted})
with probability at least $1-5(n_1+n_2)^{-10}$
provided that for all $ i,j $, $ p_{ij}\ge \frac{1}{\min\{n_1,n_2\}^{10}} $ and
\begin{equation}
p_{ij}\!\ge \!c_{0}\!\left(\frac{R_{i}^{2}}{\sum_{i'=1}^{\left\lfloor n_1/(\mu_{0}r)\right\rfloor }\!R_{i'}^{2}}\!+\!\frac{C_{j}^{2}}{\sum_{j'=1}^{\left\lfloor n_2/(\mu_{0}r)\right\rfloor }\!C_{j'}^{2}}\right)\!\log^{2}n.\label{eq:general_cond}
\end{equation}
\end{theorem}
We prove this theorem by drawing a connection between the weighted nuclear norm and the leverage scores~\eqref{eq:local_incoherence}. Define the scaled matrix $\bar{M}:=RMC$. Observe that the program~(\ref{eq:weighted}) is equivalent to first solving the following \emph{unweighted} problem with scaled observations 
\begin{equation}\label{eq:weighted2}
\begin{aligned}
\bar{X}=\arg\min_X & \;\; \left\Vert X\right\Vert _{*}\\
\textrm{s.t.} & \;\; X_{ij}=\bar{M}_{ij}, \text{ for }(i,j)\in\Omega,
\end{aligned}
\end{equation}
and then rescaling the solution $ \bar{X} $ to return $\hat{X}=R^{-1}\bar{X}C^{-1}$. In other words, through the weighted nuclear norm, we convert the problem of completing $ M $ to that of completing $ \bar{M} $. This leads to the following  observation, which underlines the proof of Theorem~\ref{cor:general_weighted}:
\begin{quote}
{\it If we can choose the weights $ R $ and $ C $ such that the leverage scores of $ \bar{M} $, denoted as $\bar{\mu}_i:=\mu_i(\bar{M}), \bar{\nu}_{j}:=\nu_i(\bar{M}),i,j\in[n] $, are aligned with the non-uniform observations in a way that roughly satisfies the relation~\eqref{eq:coherent_dist}, then we gain in sample complexity compared to the unweighted approach.}
\end{quote}   
We now quantify this more precisely for a particular class of matrix completion problems. 

\paragraph*{Comparison to unweighted nuclear norm.}

Suppose $ n_1=n_2=n $ and the sampling probabilities have a product form: $p_{ij}=p_{i}^{r}p_{j}^{c}$, with $p_{1}^{r}\le p_{2}^{r}\le\cdots\le p_{n}^{r}$ and $p_{1}^{c}\le p_{2}^{c}\le\cdots\le p_{n}^{c}$.
If we choose $R_{i}=\sqrt{\frac{1}{n}p_{i}^{r}\sum_{j'}p_{j'}^{c}}$ and $C_{j}=\sqrt{\frac{1}{n}p_{j}^{c}\sum_{i'}p_{i'}^{r}}$---which is suggested by  the condition~(\ref{eq:general_cond})---Theorem~\ref{cor:general_weighted} asserts that the following is sufficient for recovery of $ M $ with high probability:
\begin{equation}\label{eq:weighted_cond}
p_{j}^{c}\!\!\!\!\!\!\sum_{i=1}^{\left\lfloor n/(\mu_{0}r)\right\rfloor }\!\!\!\!\!\!p_{i}^{r}\gtrsim\log^{2}n,\forall j;\quad 
p_{i}^{r}\!\!\!\!\!\!\sum_{j=1}^{\left\lfloor n/(\mu_{0}r)\right\rfloor }\!\!\!\!\!\!p_{j}^{c}\gtrsim\log^{2}n,\forall i.
\end{equation}
We can compare this condition to that required by unweighted nuclear norm minimization: by Theorem~\ref{thm:random}, the latter requires 
\[
p_{i}^{r}p_{j}^{c}\gtrsim\frac{\mu_{0}r}{n}\log^{2}n, \quad \forall i,j.
\]
That is, the weighted nuclear norm approach succeeds under much less restrictive conditions. In particular, the unweighted approach imposes a condition on the \emph{least} sampled row and \emph{least} sampled column, whereas the condition~\eqref{eq:weighted_cond} shows that the weighted approach can use the heavily sampled rows/columns to assist the less sampled. This benefit is most significant precisely when the observations are very non-uniform. Indeed, the advantage of the weighted formulation is empirically observed in~\cite{salakhutdinov2010collaborative,foygel2011learning} with the weights $ R $ and $ C $ chosen as above using the empirical sampling distribution.

We remark that Theorem~\ref{cor:general_weighted} is the first exact recovery guarantee for weighted nuclear norm minimization. It provides a theoretical explanation, complementary to those in~\cite{salakhutdinov2010collaborative,foygel2011learning,negahban2012restricted}, for why the weighted approach is advantageous over the unweighted approach for non-uniform observations. It also serves as a testament to the power of Theorem~\ref{thm:random} as a general result on the relationship between sampling and the coherence/leverage score structure of a matrix.

\acks{We would like to thank Petros Drineas, Michael Mahoney and Aarti Singh for helpful discussions.  R. Ward was supported in part by an NSF CAREER award, AFOSR Young Investigator Program award, and ONR Grant N00014-12-1-0743. S. Sanghavi would like to acknowledge support from the NSF, ARO and DTRA.}

\clearpage

\appendix

\section{Proof of Theorem~\ref{thm:random}\label{sec:Proofs}}

We prove our main result Theorem~\ref{thm:random} in this
section. The overall outline of the proof is a standard convex duality argument. The main difference in establishing our results is that, while other proofs relied on bounding the $ \ell_\infty $ norm of certain random matrices, we instead bound the weighted $ \ell_{\infty,2} $, norm (to be defined).

The high level roadmap of the proof is a standard one: by
convex analysis, to show that $M$ is the unique optimal solution
to~(\ref{eq:method}), it suffices to construct a \emph{dual certificate
}$Y$ obeying certain sub-gradient optimality conditions. One of the conditions
requires the spectral norm $\left\Vert Y\right\Vert $ to be small.
Previous work bounds $\left\Vert Y\right\Vert $ by the the $\ell_{\infty}$
norm $\left\Vert Y'\right\Vert _{\infty}:=\sum_{i,j}\left|Y'_{ij}\right|$
of a certain matrix $Y'$, which gives rise to the standard and joint incoherence conditions involving uniform bounds by $ \mu_0 $ and $ \mu_{str} $. Here, we derive a new bound
using the weighted $\ell_{\infty,2}$ norm of $ Y' $, which is
the maximum of the weighted row and column norms of $Y'$. These bounds lead to a tighter bound of $\left\Vert Y\right\Vert $ and hence less restrictive conditions for matrix completion. 

We now turn to the details. To simplify the notion, we prove the results for square matrices ($n_{1}=n_{2}=n$). The results for non-square matrices are proved in exactly the same fashion. In the sequel by \emph{with high
probability} (\emph{w.h.p.}) we mean with probability at least $1-n^{-20}$.  In the proof we will show that each random event holds with high probability, and since there are no more than $ 5n^{10} $ such events, it follows from the union bound that all the events simultaneously hold with probability at least $ 1-5n^{-10} $, which is the success probability in the statement of Theorem~\ref{thm:random}. 

A few additional notations are needed. We drop the dependence of $ \mu_i(M)$ and $\nu_j(M) $  on $ M $ and simply use $ \mu_i$ and $\nu_j $.
We use $c$ and its derivatives ($c',c_{0}$, etc.) for universal positive
constants, which may differ from place to place.
The inner product between two matrices is given by $\left\langle Y,Z\right\rangle =\mbox{trace}(Y^{\top}Z)$.
Recall that $U$ and $V$ are the left and right singular vectors
of the underlying matrix $M$. We need several standard projection
operators for matrices. The projections $P_{T}$ and $P_{T^{\bot}}$
are given by 
\[
P_{T}(Z):=UU^{\top}Z+ZVV^{\top}-UU^{\top}VZZ^{\top}
\]
and $P_{T^{\bot}}(Z):=Z-P_{T}(Z).$ $P_{\Omega}(Z)$ is the matrix
with $\left(P_{\Omega}(Z)\right)_{ij}=Z_{ij}$ if $(i,j)\in\Omega$
and zero otherwise, and $P_{\Omega^{c}}(Z):=Z-P_{\Omega}(Z)$. As
usual, $\left\Vert z\right\Vert _{2}$ is the $\ell_{2}$ norm of
the vector $z$, and $\left\Vert Z\right\Vert _{F}$ and $\left\Vert Z\right\Vert $
are the Frobenius norm and spectral norm of the matrix $Z$, respectively.
For a linear operator $\mathcal{A}$ on matrices, its operator norm
is defined as $\left\Vert \mathcal{A}\right\Vert _{op}=\sup_{X\in\mathbb{R}^{n\times n}}\left\Vert \mathcal{A}(X)\right\Vert _{F}/\left\Vert X\right\Vert _{F}.$
For each $1\le i,j\le n$, we define the random variable $\delta_{ij}:=\mathbb{I}\left((i,j)\in\Omega\right)$,
where $\mathbb{I}(\cdot)$ is the indicator function. The matrix operator
$R_{\Omega}:\mathbb{R}^{n\times n}\mapsto\mathbb{R}^{n\times n}$
is defined as 
\begin{equation}
R_{\Omega}(Z)=\sum_{i,j}\frac{1}{p_{ij}}\delta_{ij}\left\langle e_{i}e_{j}^{\top},Z\right\rangle e_{i}e_{j}^{\top}.
\end{equation}

\paragraph*{Optimality Condition.}

Following our proof roadmap, we now state a sufficient condition for
$M$ to be the unique optimal solution to the optimization problem
(\ref{eq:method}). This is the content of Proposition~\ref{prop:opt_cond}
below (proved in Section~\ref{sec:proof_opt_cond} to follow). 
\begin{proposition}
\label{prop:opt_cond} Suppose $p_{ij}\ge\frac{1}{n^{10}}$. The matrix
$M$ is the unique optimal solution to~(\ref{eq:method}) if the
following conditions hold. 
\begin{enumerate}
\item $\left\Vert P_{T}R_{\Omega}P_{T}-P_{T}\right\Vert _{op}\le\frac{1}{2}.$
\item There exists a dual certificate $Y\in\mathbb{R}^{n\times n}$ which
satisfies $P_{\Omega}(Y)=Y$ and

\begin{enumerate}
\item $\left\Vert P_{T}(Y)-UV^{\top}\right\Vert _{F}\le\frac{1}{4n^{5}},$
\item $\left\Vert P_{T^{\bot}}(Y)\right\Vert \le\frac{1}{2}$.
\end{enumerate}
\end{enumerate}
\end{proposition}

\paragraph*{Validating the Optimality Condition.}

We begin by proving that Condition 1 in Proposition~\ref{prop:opt_cond}
is satisfied under the conditions of Theorem~\ref{thm:random}. This
is done in the following lemma (proved in Section~\ref{sec:proof_tech} to follow).
The lemma shows that $R_{\Omega}$ is close to the identity operator
on $T$.
\begin{lemma}
\label{lem:op}If $p_{ij}\ge \min\{c_{0}\frac{(\mu_{i}+\nu_{j})r}{n}\log n,1\}$
for all $(i,j)$ and a sufficiently large $c_{0}$, then w.h.p.
\begin{equation}
\left\Vert P_{T}R_{\Omega}P_{T}-P_{T}\right\Vert _{op}\le\frac{1}{2}.
\end{equation}

\end{lemma}

\paragraph*{Constructing the Dual Certificate.}

It remains to construct a matrix $Y$ (the dual certificate) that
satisfies the condition 2 in Proposition~\ref{prop:opt_cond}. We
do this using the golfing scheme~\citep{gross2011recovering,candes2009robustPCA}.
Set $k_{0}=20\log n$. Suppose the set $\Omega$ of observed elements
is generated from $\Omega=\bigcup_{k=1}^{k_{0}}\Omega_{k}$, where
for each $k=1,\ldots,k_{0}$ and matrix index $(i,j)$, $\mathbb{P}\left[(i,j)\in\Omega_{k}\right]=q_{ij}:=1-(1-p_{ij})^{1/k_{0}}$
independent of all others. Clearly this is equivalent to the original
Bernoulli sampling model. Let $W_{0}:=0$ and for $k=1,\ldots,k_{0},$
\begin{equation}
W_{k}:=W_{k-1}+R_{\Omega_{k}}P_{T}(UV^{\top}-P_{T}W_{k-1}),
\end{equation}
where the operator $R_{\Omega_{k}}$ is given by 
\[
R_{\Omega_{k}}(Z)=\sum_{i,j}\frac{1}{q_{ij}}\mathbb{I}\left((i,j)\in\Omega_{k}\right)\left\langle e_{i}e_{j}^{\top},Z\right\rangle e_{i}e_{j}^{\top}.
\]
The dual certificate is given $Y:=W_{k_{0}}.$ Clearly $P_{\Omega}(Y)=Y$
by construction. The proof of Theorem~\ref{thm:random} is completed
if we show that under the condition in theorem, $Y$ satisfies Conditions
2(a) and 2(b) in Proposition~\ref{prop:opt_cond} w.h.p.

\paragraph*{Concentration Properties}

The key step in our proof is to show that $Y$ satisfies Condition
2(b) in Proposition~\ref{prop:opt_cond}, i.e., we need to bound $\left\Vert P_{T^{\bot}}(Y)\right\Vert $
. Here our proof departs from existing ones, as we establish concentration
bounds on this quantity in terms of (an appropriately weighted version
of) the $\ell_{\infty,2}$ norm, which we now define. The $\mu(\infty,2)$-norm
of a matrix $Z\in\mathbb{R}^{n\times n}$ is defined as 
\begin{align*}
\left\Vert Z\right\Vert _{\mu(\infty,2)} & :=\max\left\{ \max_{i}\sqrt{\frac{n}{\mu_{i}r}\sum_{b}Z_{ib}^{2}},\max_{j}\sqrt{\frac{n}{\nu_{j}r}\sum_{a}Z_{aj}^{2}}\right\} ,
\end{align*}
which is the maximum of the weighted column and row norms of $Z$.
We also need the $\mu(\infty)$-norm of $Z$, which is a weighted
version of the matrix $\ell_{\infty}$ norm. This is given as 
\[
\left\Vert Z\right\Vert _{\mu(\infty)}:=\max_{i,j}\left|Z_{ij}\right|\sqrt{\frac{n}{\mu_{i}r}}\sqrt{\frac{n}{\nu_{j}r}}.
\]
which is the weighted element-wise magnitude of $Z$. We now state three
new lemmas concerning the concentration properties of these norms.
The first lemma is crucial to our proof; it bounds the spectral norm
of $\left(R_{\Omega}-I\right)Z$ in terms of the $\mu(\infty,2)$
and $\mu(\infty)$ norms of $Z$. This obviates intermediate lemmas required previous
approaches~\citep{candes2010power,gross2011recovering,recht2009simpler,keshavan2010matrix} 
which use the $\ell_{\infty}$ norm of $Z$.
\begin{lemma}
\label{lem:op_mu} Suppose $Z$ is a fixed $n\times n$ matrix. For
some universal constant $c>1$, we have w.h.p.
\[
\left\Vert \left(R_{\Omega}-I\right)Z\right\Vert \le c\left(\max_{i,j}\left|\frac{Z_{ij}}{p_{ij}}\right|\log n+\sqrt{\max\left\{ \max_{i}\sum_{j=1}^{n}\frac{Z_{ij}^{2}}{p_{ij}},\max_{j}\sum_{i=1}^{n}\frac{Z_{ij}^{2}}{p_{ij}}\right\} \log n}\right).
\]
If $p_{ij}\ge \min\{c_{0}\frac{(\mu_{i}+\nu_{j})r}{n}\log n,1\}$ for all $(i,j)$,
then we further have w.h.p. $$\left\Vert \left(R_{\Omega}-I\right)Z\right\Vert \le \frac{c}{\sqrt{c_{0}}}\left(\left\Vert Z\right\Vert _{\mu(\infty)}+\left\Vert Z\right\Vert _{\mu(\infty,2)}\right).$$
\end{lemma}
The next two lemmas further control the $\mu(\infty,2)$ and $\mu(\infty)$
norms of a matrix after random projections.
\begin{lemma}
\label{lem:muinf2}Suppose $Z$ is a fixed $n\times n$ matrix. If
$p_{ij}\ge \min\{c_{0}\frac{(\mu_{i}+\nu_{j})r}{n}\log n,1\}$ for all $i,j$
and sufficiently large $c_{0}$, then w.h.p. 
\[
\left\Vert (P_{T}R_{\Omega}-P_{T})Z\right\Vert _{\mu(\infty,2)}\le\frac{1}{2}\left(\left\Vert Z\right\Vert _{\mu(\infty)}+\left\Vert Z\right\Vert _{\mu(\infty,2)}\right)
\]

\end{lemma}

\begin{lemma}
\label{lem:muinf} Suppose $Z$ is a fixed $n\times n$ matrix. If
$p_{ij}\ge \min\{c_{0}\frac{(\mu_{i}+\nu_{j})r}{n}\log n,1\}$ for all $i,j$
and $c_{0}$ sufficiently large, then w.h.p.
\[
\left\Vert \left(P_{T}R_{\Omega}-P_{T}\right)Z\right\Vert _{\mu(\infty)}\le\frac{1}{2}\left\Vert Z\right\Vert _{\mu(\infty)}.
\]

\end{lemma}
We prove Lemmas~\ref{lem:op_mu}--\ref{lem:muinf} in Section~\ref{sec:proof_tech}.
Equipped with the three lemmas above, we are now ready to validate
that $Y$ satisfies Condition 2 in Proposition~\ref{prop:opt_cond}.

\paragraph{Validating Condition 2(a):}

Set $\Delta_{k}=UV^{\top}-P_{T}(W_{k})$ for $k=1,\ldots,k_{0}$.
By definition of $W_{k}$, we have 
\begin{align}
\Delta_{k} & =\left(P_{T}-P_{T}R_{\Omega_{k}}P_{T}\right)\Delta_{k-1}.\label{eq:recursion}
\end{align}
Note that $\Omega_{k}$ is independent of $\Delta_{k-1}$ and $q_{ij}\ge p_{ij}/k_{0}\ge c_{0}'(\mu_{i}+\nu_{j})r\log(n)/n$
under the condition in Theorem~\ref{thm:random}. Applying Lemma~\ref{lem:op}
with $\Omega$ replaced by $\Omega_{k}$ , we obtain that w.h.p.
\[
\left\Vert \Delta_{k}\right\Vert _{F}\le\left\Vert P_{T}-P_{T}R_{\Omega_{k}}P_{T}\right\Vert \left\Vert \Delta_{k-1}\right\Vert _{F}\le\frac{1}{2}\left\Vert \Delta_{k-1}\right\Vert _{F}.
\]
Applying the above inequality recursively with $k=k_{0,}k_{0}-1,\ldots,1$
gives
\[
\left\Vert P_{T}(Y)-UV^{\top}\right\Vert _{F}=\left\Vert \Delta_{k_{0}}\right\Vert _{F}\le\left(\frac{1}{2}\right)^{k_{0}}\left\Vert UV^{\top}\right\Vert _{F}\le\frac{1}{4n^{6}}\cdot\sqrt{r}\le\frac{1}{4n^{5}}.
\]

\paragraph{Validating Condition 2(b): }

By definition, $Y$ can be rewritten as $Y=\sum_{k=1}^{k_{0}}R_{\Omega_{k}}P_{T}\Delta_{k-1}.$
It follows that 
\[
\left\Vert P_{T^{\bot}}(Y)\right\Vert =\left\Vert P_{T^{\bot}}\sum_{k=1}^{k_{0}}\left(R_{\Omega_{k}}P_{T}-P_{T}\right)\Delta_{k-1}\right\Vert \le\sum_{k=1}^{k_{0}}\left\Vert \left(R_{\Omega_{k}}-I\right)\Delta_{k-1}\right\Vert .
\]
We apply Lemma~\ref{lem:op_mu} with $\Omega$ replaced by $\Omega_{k}$
to each summand in the last RHS to obtain w.h.p. 
\begin{align}
\left\Vert P_{T^{\bot}}(Y)\right\Vert  & \le\frac{c}{\sqrt{c_{0}}}\sum_{k=1}^{k_{0}}\left\Vert \Delta_{k-1}\right\Vert _{\mu(\infty)}+\frac{c}{\sqrt{c_{0}}}\sum_{k=1}^{k_{0}}\left\Vert \Delta_{k-1}\right\Vert _{\mu(\infty,2)}.
\end{align}
We bound each summand in the last RHS. Applying $(k-1)$ times~(\ref{eq:recursion})
and Lemma~\ref{lem:muinf} (with $\Omega$ replaced by $\Omega_{k}$),
we have w.h.p.
\begin{align*}
\left\Vert \Delta_{k-1}\right\Vert _{\mu(\infty)} & =\left\Vert \left(P_{T}-P_{T}R_{\Omega_{k-1}}P_{T}\right)\Delta_{k-2}\right\Vert _{\mu(\infty)}\le\left(\frac{1}{2}\right)^{k-1}\left\Vert UV^{\top}\right\Vert _{\mu(\infty)}.
\end{align*}
for each $k$. Similarly, repeatedly applying~(\ref{eq:recursion}),
Lemma~\ref{lem:muinf2} and the inequality we just proved above, we
obtain w.h.p.
\begin{align}
 & \left\Vert \Delta_{k-1}\right\Vert _{\mu(\infty,2)}\\
= & \left\Vert \left(P_{T}-P_{T}R_{\Omega_{k-1}}P_{T}\right)\Delta_{k-2}\right\Vert _{\mu(\infty,2)}\\
\le & \frac{1}{2}\left\Vert \Delta_{k-2}\right\Vert _{\mu(\infty)}+\frac{1}{2}\left\Vert \Delta_{k-2}\right\Vert _{\mu(\infty,2)}\\
\le & \left(\frac{1}{2}\right)^{k-1}\left\Vert UV^{\top}\right\Vert _{\mu(\infty)}+\frac{1}{2}\left\Vert \Delta_{k-2}\right\Vert _{\mu(\infty,2)}\\
\le & k\left(\frac{1}{2}\right)^{k-1}\left\Vert UV^{\top}\right\Vert _{\mu(\infty)}+\left(\frac{1}{2}\right)^{k-1}\left\Vert UV\right\Vert _{\mu(\infty,2)}.
\end{align}
It follows that w.h.p. 
\begin{align}
\left\Vert P_{T^{\bot}}(Y)\right\Vert  & \le\frac{c}{\sqrt{c_{0}}}\sum_{k=1}^{k_{0}}(k+1)\left(\frac{1}{2}\right)^{k-1}\left\Vert UV^{\top}\right\Vert _{\mu(\infty)}+\frac{c}{\sqrt{c_{0}}}\sum_{k=1}^{k_{0}}\left(\frac{1}{2}\right)^{k-1}\left\Vert UV^{\top}\right\Vert _{\mu(\infty,2)}\\
 & \le\frac{6c}{\sqrt{c_{0}}}\left\Vert UV^{\top}\right\Vert _{\mu(\infty)}+\frac{2c}{\sqrt{c_{0}}}\left\Vert UV^{\top}\right\Vert _{\mu(\infty,2)}.
\end{align}
Note that for all $(i,j)$, we have $\left|\left(UV^{\top}\right)_{ij}\right|=\left|e_{i}^{\top}UV^{\top}e_{j}\right|\le\sqrt{\frac{\mu_{i}r}{n}}\sqrt{\frac{\nu_{j}r}{n}}$,
$\left\Vert e_{i}^{\top}UV^{\top}\right\Vert _{2}=\sqrt{\frac{\mu_{i}r}{n}}$
and $\left\Vert UV^{\top}e_{j}\right\Vert _{2}=\sqrt{\frac{\nu_{j}r}{n}}$.
Hence $\left\Vert UV^{\top}\right\Vert _{\mu(\infty)}\le1$ and $\left\Vert UV^{\top}\right\Vert _{\mu(\infty,2)}=1$.
We conclude that 
\[
\left\Vert P_{T^{\bot}}(Y)\right\Vert \le\frac{6c}{\sqrt{c_{0}}}+\frac{2c}{\sqrt{c_{0}}}\le\frac{1}{2}
\]
provided that the constant $c_{0}$ in Theorem~\ref{thm:random} is
sufficiently large. This completes the proof of Theorem~\ref{thm:random}.

\subsection{Proof of Optimality Condition (Proposition~\ref{prop:opt_cond})\label{sec:proof_opt_cond}}
\begin{proof}
Consider any feasible solution $X$ to~(\ref{eq:method}) with $P_{\Omega}(X)=P_{\Omega}(M)$.
Let $G$ be an $n\times n$ matrix which satisfies $\left\Vert P_{T^{\bot}}G\right\Vert =1$,
and $\left\langle P_{T^{\bot}}G,P_{T^{\bot}}(X-M)\right\rangle =\left\Vert P_{T^{\bot}}(X-M)\right\Vert _{*}$.
Such $G$ always exists by duality between the nuclear norm and spectral
norm. Because $UV^{\top}+P_{T^{\bot}}G$ is a sub-gradient of the
function $f(Z)=\left\Vert Z\right\Vert _{*}$ at $Z=M$, we have 
\begin{equation}
\left\Vert X\right\Vert _{*}-\left\Vert M\right\Vert _{*} 
\ge \left\langle UV^{\top}+P_{T^{\bot}}G,X-M\right\rangle .
\end{equation}
But $\left\langle Y,X-M\right\rangle =\left\langle P_{\Omega}(Y),P_{\Omega}(X-M)\right\rangle =0$
since $P_{\Omega}(Y)=Y$. It follows that 
\begin{align*}
\left\Vert X\right\Vert _{*}-\left\Vert M\right\Vert _{*} & \ge\left\langle UV^{\top}+P_{T^{\bot}}G-Y,X-M\right\rangle \\
 & =\left\Vert P_{T^{\bot}}(X-M)\right\Vert _{*}+\left\langle UV^{\top}-P_{T}Y,X-M\right\rangle -\left\langle P_{T^{\bot}}Y,X-M\right\rangle \\
 & \ge\left\Vert P_{T^{\bot}}(X-M)\right\Vert _{*}-\left\Vert UV^{\top}-P_{T}Y\right\Vert _{F}\left\Vert P_{T}(X-M)\right\Vert _{F}-\left\Vert P_{T^{\bot}}Y\right\Vert \left\Vert P_{T^{\bot}}(X-M)\right\Vert _{*}\\
 & \ge\frac{1}{2}\left\Vert P_{T^{\bot}}(X-M)\right\Vert _{*}-\frac{1}{4n^{5}}\left\Vert P_{T}(X-M)\right\Vert _{F},
\end{align*}
where in the last inequality we use conditions 1 and 2 in the proposition.
Using Lemma~\ref{lem:mini_lemma} below, we obtain
\begin{align*}
\left\Vert X\right\Vert _{*}-\left\Vert M\right\Vert _{*} & \ge\frac{1}{2}\left\Vert P_{T^{\bot}}(X-M)\right\Vert _{*}-\frac{1}{4n^{5}}\cdot\sqrt{2}n^{5}\left\Vert P_{T^{\bot}}(X-M)\right\Vert _{*}>\frac{1}{8}\left\Vert P_{T^{\bot}}(X-M)\right\Vert _{*}.
\end{align*}
The RHS is strictly positive for all $X$ with $P_{\Omega}(X-M)=0$
and $X\neq M$. Otherwise we must have $P_{T}(X-M)=X-M$ and $P_{T}P_{\Omega}P_{T}(X-M)=0$,
contradicting the assumption $\left\Vert P_{T}R_{\Omega}P_{T}-P_{T}\right\Vert _{op}\le\frac{1}{2}$.
This proves that $M$ is the unique optimum.\end{proof}
\begin{lemma}
\label{lem:mini_lemma}If $p_{ij}\ge\frac{1}{n^{10}}$ for all $(i,j)$
and $\left\Vert P_{T}R_{\Omega}P_{T}-P_{T}\right\Vert _{op}\le\frac{1}{2}$,
then we have 
\begin{equation}
\left\Vert P_{T}Z\right\Vert _{F}\le\sqrt{2}n^{5}\left\Vert P_{T^{\bot}}(Z)\right\Vert _{*},\forall Z\in\{Z':P_{\Omega}(Z')=0\}.
\end{equation}
\end{lemma}
\begin{proof}
Define the operator $R_{\Omega}^{1/2}:\mathbb{R}^{n\times n}\mapsto\mathbb{R}^{n\times n}$
by 
\[
R_{\Omega}^{1/2}(Z):=\sum_{i,j}\frac{1}{\sqrt{p_{ij}}}\delta_{ij}\left\langle e_{i}e_{j}^{\top},Z\right\rangle e_{i}e_{j}^{\top}.
\]
Note that $R_{\Omega}^{1/2}$ is self-adjoint and satisfies $R_{\Omega}^{1/2}R_{\Omega}^{1/2}=R_{\Omega}$.
Hence we have 
\begin{align*}
\left\Vert R_{\Omega}^{1/2}P_{T}(Z)\right\Vert _{F} & =\sqrt{\left\langle P_{T}R_{\Omega}P_{T}Z,P_{T}Z\right\rangle }\\
 & =\sqrt{\left\langle \left(P_{T}R_{\Omega}P_{T}-P_{T}\right)Z,P_{T}(Z)\right\rangle +\left\langle P_{T}(Z),P_{T}(Z)\right\rangle }\\
 & \ge\sqrt{\left\Vert P_{T}(Z)\right\Vert _{F}^{2}-\left\Vert P_{T}R_{\Omega}P_{T}-P_{T}\right\Vert \left\Vert P_{T}(Z)\right\Vert _{F}^{2}}\\
 & \ge\frac{1}{\sqrt{2}}\left\Vert P_{T}(Z)\right\Vert _{F},
\end{align*}
where the last inequality follows from the assumption $\left\Vert P_{T}R_{\Omega}P_{T}-P_{T}\right\Vert _{op}\le\frac{1}{2}$.
On the other hand, $P_{\Omega}(Z)=0$ implies $R_{\Omega}^{1/2}(Z)=0$
and thus 
\[
\left\Vert R_{\Omega}^{1/2}P_{T}(Z)\right\Vert _{F}=\left\Vert R_{\Omega}^{1/2}P_{T^{\bot}}(Z)\right\Vert _{F}\le\left(\max_{i,j}\frac{1}{\sqrt{p_{ij}}}\right)\left\Vert P_{T^{\bot}}(Z)\right\Vert _{F}\le n^{5}\left\Vert P_{T^{\bot}}(Z)\right\Vert _{F}.
\]
Combining the last two display equations gives
\[
\left\Vert P_{T}(Z)\right\Vert _{F}\le\sqrt{2}n^{5}\left\Vert P_{T^{\bot}}(Z)\right\Vert _{F}\le\sqrt{2}n^{5}\left\Vert P_{T^{\bot}}(Z)\right\Vert _{*}.
\]

\end{proof}

\subsection{Proof of Technical Lemmas\label{sec:proof_tech}}

We prove the four technical lemmas that are used in the proof of our
main theorem. The proofs use the matrix Bernstein inequality given
as Theorem~\ref{lem:matrix_bernstein} in Section~\ref{sec:bernstein}.
We also make frequent use of the following facts: for all $i$ and
$j$, we have $\max\left\{ \frac{\mu_{i}r}{n},\frac{\nu_{j}r}{n}\right\} \le1$
and 
\begin{equation}
\frac{(\mu_{i}+\nu_{j})r}{n}\ge\left\Vert P_{T}(e_{i}e_{j}^{\top})\right\Vert _{F}^{2}.\label{eq:basic_inequality}
\end{equation}
We also use the shorthand $a\wedge b:=\min\{a,b\}.$

\subsubsection{Proof of Lemma~\ref{lem:op}}

For any matrix $Z$, we can write 
\[
(P_{T}R_{\Omega}P_{T}-P_{T})(Z)=\sum_{i,j}\left(\frac{1}{p_{ij}}\delta_{ij}-1\right)\left\langle e_{i}e_{j}^{\top},P_{T}(Z)\right\rangle P_{T}(e_{i}e_{j}^{\top})=:\sum_{i,j}\mathcal{S}_{ij}(Z).
\]
Note that $\mathbb{E}\left[\mathcal{S}_{ij}\right]=0$ and $\mathcal{S}_{ij}$'s
are independent of each other. For all $Z$ and $(i,j)$, we have
$\mathcal{S}_{ij}=0$ if $p_{ij}=1$. On the other hand, when $p_{ij}\ge c_{0}\frac{(\mu_{i}+\nu_{j})r\log n}{n}$,
then it follows from~(\ref{eq:basic_inequality}) that 
\[
\left\Vert \mathcal{S}_{ij}(Z)\right\Vert _{F}\le\frac{1}{p_{ij}}\left\Vert P_{T}(e_{i}e_{j}^{\top})\right\Vert _{F}^{2}\left\Vert Z\right\Vert _{F}\le\max_{i,j}\left\{ \frac{1}{p_{ij}}\frac{(\mu_{i}+\nu_{j})r}{n}\right\} \left\Vert Z\right\Vert _{F}\le\frac{1}{c_{0}\log n}\left\Vert Z\right\Vert _{F}.
\]
Putting together, we have that $\left\Vert \mathcal{S}_{ij}\right\Vert \le\frac{1}{c_{0}\log n}$
under the condition of the lemma. On the other hand, we have 
\begin{align*}
\left\Vert \sum_{i,j}\mathbb{E}\left[\mathcal{S}_{ij}^{2}(Z)\right]\right\Vert _{F} & =\left\Vert \sum_{i,j}\mathbb{E}\left[\left(\frac{1}{p_{ij}}\delta_{ij}-1\right)^{2}\left\langle e_{i}e_{j}^{\top},P_{T}(Z)\right\rangle \left\langle e_{i}e_{j}^{\top},P_{T}(e_{i}e_{j}^{\top})\right\rangle P_{T}(e_{i}e_{j}^{\top})\right]\right\Vert _{F}\\
 & \le\left(\max_{i,j}\frac{1-p_{ij}}{p_{ij}}\left\Vert P_{T}(e_{i}e_{j}^{\top})\right\Vert _{F}^{2}\right)\left\Vert \sum_{i,j}\left\langle e_{i}e_{j}^{\top},P_{T}(Z)\right\rangle P_{T}(e_{i}e_{j}^{\top})\right\Vert _{F}\\
 & \le\max_{i,j}\left\{ \frac{1-p_{ij}}{p_{ij}}\frac{(\mu_{i}+\nu_{j})r}{n}\right\} \left\Vert P_{T}(Z)\right\Vert _{F},
\end{align*}
This implies $\left\Vert \sum_{i,j}\mathbb{E}\left[\mathcal{S}_{ij}^{2}\right]\right\Vert \le\frac{1}{c_{0}\log n}$
under the condition of the lemma. Applying the Matrix Bernstein inequality
(Theorem~\ref{lem:matrix_bernstein}), we obtain $\left\Vert P_{T}R_{\Omega}P_{T}-P_{T}\right\Vert =\left\Vert \sum_{i,j}\mathcal{S}_{ij}\right\Vert \le\frac{1}{2}$
w.h.p. for sufficiently large $c_{0}$.

\subsubsection{Proof of Lemma~\ref{lem:op_mu}}

We can write $\left(R_{\Omega}-I\right)Z$ as the sum of independent
matrices:
\[
\left(R_{\Omega}-I\right)Z=\sum_{i,j}\left(\frac{1}{p_{ij}}\delta_{ij}-1\right)Z_{ij}e_{i}e_{j}^{\top}=:\sum_{i,j}S_{ij}.
\]
Note that $\mathbb{E}[S_{ij}]=0$. For all $(i,j)$, we have $S_{ij}=0$
if $p_{ij}=1$, and 
\[
\left\Vert S_{ij}\right\Vert \le\frac{1}{p_{ij}}\left|Z_{ij}\right|.
\]
Moreover, 
\begin{align*}
\left\Vert \mathbb{E}\left[\sum_{i,j}S_{ij}^{\top}S_{ij}\right]\right\Vert  & =\left\Vert \sum_{i,j}Z_{ij}^{2}e_{i}e_{j}^{\top}e_{j}e_{i}^{\top}\mathbb{E}\left(\frac{1}{p_{ij}}\delta_{ij}-1\right)^{2}\right\Vert =\max_{i}\sum_{j=1}^{n}\frac{1-p_{ij}}{p_{ij}}Z_{ij}^{2}.
\end{align*}
The quantity $\left\Vert \mathbb{E}\left[\sum_{i,j}S_{ij}S_{ij}^{\top}\right]\right\Vert $
is bounded by $\max_{j}\sum_{i=1}^{n}(1-p_{ij})Z_{ij}^{2}/p_{ij}$
in a similar way. The first part of the lemma then follows from the
matrix Bernstein inequality (Theorem~\ref{lem:matrix_bernstein}). If $p_{ij}\ge1\wedge\frac{c_{0}(\mu_{i}+\nu_{j})r\log n}{n}\ge1\wedge2c_{0}\sqrt{\frac{\mu_{i}r}{n}\cdot\frac{\nu_{j}r}{n}}\log n$,
we have for all $i$ and $j$,
\begin{align*}
\left\Vert S_{ij}\right\Vert \log n
&\le\left(1-\mathbb{I}(p_{ij}=1)\right)\frac{1}{p_{ij}}\left|Z_{ij}\right|\log n
\le\frac{1}{c_{0}}\left\Vert Z\right\Vert _{\mu(\infty)},\\
\sum_{i=1}^{n}\frac{1-p_{ij}}{p_{ij}}Z_{ij}^{2}\log n
&\le\frac{1}{c_{0}}\left\Vert Z\right\Vert _{\mu(\infty,2)}^{2},\\
\sum_{j=1}^{n}\frac{1-p_{ij}}{p_{ij}}Z_{ij}^{2}\log n
&\le\frac{1}{c_{0}}\left\Vert Z\right\Vert _{\mu(\infty,2)}^{2}.
\end{align*} 
The second part of the lemma follows again from applying the matrix
Bernstein inequality.

\subsubsection{Proof of Lemma~\ref{lem:muinf2}}

Let $X=\left(P_{T}R_{\Omega}-P_{T}\right)Z$. By definition we have
$$\left\Vert X\right\Vert _{\mu(\infty,2)}=\max_{a,b}\left\{ \sqrt{\frac{n}{\mu_{a}r}}\left\Vert X_{a\cdot}\right\Vert _{2},\sqrt{\frac{n}{\nu_{b}r}}\left\Vert X_{\cdot b}\right\Vert _{2}\right\}, $$
where $X_{a\cdot}$ and $X_{\cdot b}$ are the $a$-th row and $b$-th
column of of $X$, respectively. We bound each term in the maximum.
Observe that $\sqrt{\frac{n}{\nu_{b}r}}X_{\cdot b}$ can be written
as the sum of independent column vectors:
\begin{align*}
\sqrt{\frac{n}{\nu_{b}r}}X_{\cdot b} & =\sum_{i,j}\left(\frac{1}{p_{ij}}\delta_{ij}-1\right)Z_{ij}\left(P_{T}(e_{i}e_{j}^{\top})e_{b}\right)\sqrt{\frac{n}{\nu_{b}r}}=:\sum_{i,j}S_{ij},
\end{align*}
where $\mathbb{E}\left[S_{ij}\right]=0$. To control $\left\Vert S_{ij}\right\Vert _{2}$
and $\left\Vert \mathbb{E}\left[\sum_{i,j}S_{ij}^{\top}S_{ij}\right]\right\Vert $,
we first need a bound for $\left\Vert P_{T}(e_{i}e_{j}^{\top})e_{b}\right\Vert _{2}$.
If $j=b$, we have
\begin{equation}
\left\Vert P_{T}(e_{i}e_{j}^{\top})e_{b}\right\Vert _{2}=\left\Vert UU^{\top}e_{i}+(I-UU^{\top})e_{i}\left\Vert V^{\top}e_{b}\right\Vert _{2}^{2}\right\Vert _{2}\le\sqrt{\frac{\mu_{i}r}{n}}+\sqrt{\frac{\nu_{b}r}{n}},\label{eq:bound31}
\end{equation}
where we use the triangle inequality and the definition of $\mu_{i}$
and $\nu_{b}$ . Similarly, if $j\neq b$, we have 
\begin{equation}
\left\Vert P_{T}(e_{i}e_{j}^{\top})e_{b}\right\Vert _{2}=\left\Vert (I-UU^{\top})e_{i}e_{j}^{\top}VV^{\top}e_{b}\right\Vert _{2}\le\left|e_{j}^{\top}VV^{\top}e_{b}\right|.\label{eq:bound32}
\end{equation}
Now note that $\left\Vert S_{ij}\right\Vert _{2}\le\left(1-\mathbb{I}(p_{ij}=1)\right)\frac{1}{p_{ij}}\left|Z_{ij}\right|\sqrt{\frac{n}{\nu_{b}r}}\left\Vert P_{T}(e_{i}e_{j}^{\top})e_{b}\right\Vert _{2}.$
Using the bounds~(\ref{eq:bound31}) and~(\ref{eq:bound32}), we obtain
that for $j=b$, 
\begin{align*}
\left\Vert S_{ij}\right\Vert _{2}&\le\left(1-\mathbb{I}(p_{ij}=1)\right)\frac{1}{p_{ib}}\left|Z_{ib}\right|\sqrt{\frac{n}{\nu_{b}r}}\cdot\left(\sqrt{\frac{\mu_{i}r}{n}}+\frac{\nu_{b}r}{n}\right)\\
&\le\frac{2}{c_{0}\sqrt{\frac{\mu_{i}r\nu_{b}r}{n^{2}}}\log n}\left|Z_{ib}\right|
\le\frac{2}{c_{0}\log n}\left\Vert Z\right\Vert _{\mu(\infty)},
\end{align*}
where we use $p_{ib}\ge1\wedge\frac{c_{0}\mu_{i}r\log n}{n}$ and
$p_{ib}\ge1\wedge c_{0}\sqrt{\frac{\mu_{i}r}{n}\frac{\nu_{b}r}{n}}\log n$
in the second inequality. For $j\neq b$, we have
\[
\left\Vert S_{ij}\right\Vert _{2}\le\left(1-\mathbb{I}(p_{ij}=1)\right)\frac{1}{p_{ij}}\left|Z_{ij}\right|\sqrt{\frac{n}{\nu_{b}r}}\cdot\sqrt{\frac{\nu_{j}r}{n}}\sqrt{\frac{\nu_{b}r}{n}}\le\frac{2}{c_{0}\log n}\left\Vert Z\right\Vert _{\mu(\infty)},
\]
where we use $p_{ij}\ge1\wedge c_{0}\sqrt{\frac{\mu_{i}r}{n}\frac{\nu_{j}r}{n}}\log n$.
We thus obtain $\left\Vert S_{ij}\right\Vert _{2}\le\frac{2}{c_{0}\log n}\left\Vert Z\right\Vert _{\mu(\infty)}$
for all $(i,j)$.

On the other hand, note that
\begin{align*}
\left|\mathbb{E}\left[{\textstyle \sum_{i,j}}S_{ij}^{\top}S_{ij}\right]\right| & =\left|{\textstyle \sum_{i,j}}\mathbb{E}\left[\left(\frac{1}{p_{ij}}\delta_{ij}-1\right){}^{2}\right]Z_{ij}^{2}\left\Vert P_{T}(e_{i}e_{j}^{\top})e_{b}\right\Vert _{2}^{2}\cdot\frac{n}{\nu_{b}r}\right|\\
 & =\left({\textstyle \sum_{j=b,i}+\sum_{j\neq b,i}}\right)\frac{1-p_{ij}}{p_{ij}}Z_{ij}^{2}\left\Vert P_{T}(e_{i}e_{j}^{\top})e_{b}\right\Vert _{2}^{2}\cdot\frac{n}{\nu_{b}r}.
\end{align*}
Applying~(\ref{eq:bound31}), we can bound the first sum by
\[
\sum_{j=b,i}\le\sum_{i}\frac{1-p_{ib}}{p_{ib}}Z_{ib}^{2}\cdot2\left(\frac{\mu_{i}r}{n}+\frac{\nu_{b}r}{n}\right)\cdot\frac{n}{\nu_{b}r}\le\frac{2}{c_{0}\log n}\frac{n}{\nu_{b}r}\left\Vert Z_{\cdot b}\right\Vert _{2}^{2}\le\frac{2}{c_{0}\log n}\left\Vert Z\right\Vert _{\mu(\infty,2)}^{2},
\]
where we use $p_{ib}\ge1\wedge\frac{c_{0}(\mu_{i}+\nu_{b})r}{n}\log n$
in the second inequality. The second sum can be bounded using~(\ref{eq:bound32}):
\begin{align*}
\sum_{j\neq b,i} & \le\sum_{j\neq b,i}\frac{1-p_{ij}}{p_{ij}}Z_{ij}^{2}\left|e_{j}^{\top}VV^{\top}e_{b}\right|^{2}\frac{n}{\nu_{b}r}\\
 & =\frac{n}{\nu_{b}r}\sum_{j\neq b}\left|e_{j}^{\top}VV^{\top}e_{b}\right|^{2}\sum_{i}\frac{1-p_{ij}}{p_{ij}}Z_{ij}^{2}\\
 & \overset{(a)}{\le}\frac{n}{\nu_{b}r}\sum_{j\neq b}\left|e_{j}^{\top}VV^{\top}e_{b}\right|^{2}\left(\frac{1}{c_{0}\log n}\sum_{i}Z_{ij}^{2}\frac{n}{\nu_{j}r}\right)\\
 & \le\left(\frac{1}{c_{0}\log n}\left\Vert Z\right\Vert _{\mu(\infty,2)}^{2}\right)\frac{n}{\nu_{b}r}\sum_{j\neq b}\left|e_{j}^{\top}VV^{\top}e_{b}\right|^{2}\\
 & \overset{(b)}{\le}\frac{1}{c_{0}\log n}\left\Vert Z\right\Vert _{\mu(\infty,2)}^{2},
\end{align*}
where we use $p_{ij}\ge1\wedge\frac{c_{0}\nu_{j}r\log n}{n}$ in $(a)$
and $\sum_{j\neq b}\left|e_{j}^{\top}VV^{\top}e_{b}\right|^{2}\le\left\Vert VV^{\top}e_{b}\right\Vert _{2}^{2}\le\frac{\nu_{b}r}{n}$
in $(b)$. Combining the bounds for the two sums, we obtain$\left\Vert \mathbb{E}\left[\sum_{i,j}S_{ij}^{\top}S_{ij}\right]\right\Vert \le\frac{3}{c_{0}\log n}\left\Vert Z\right\Vert _{\mu(\infty,2)}^{2}.$
We can bound $\left\Vert \mathbb{E}\left[\sum_{i,j}S_{ij}S_{ij}^{\top}\right]\right\Vert $
in a similar way. Applying the Matrix Bernstein inequality in Theorem
\ref{lem:matrix_bernstein}, we have w.h.p.
\[
\left\Vert \sqrt{\frac{n}{\nu_{b}r}}X_{\cdot b}\right\Vert _{2}=\left\Vert {\textstyle \sum_{i,j}}S_{ij}\right\Vert _{2}\le\frac{1}{2}\left(\left\Vert Z\right\Vert _{\mu(\infty)}+\left\Vert Z\right\Vert _{\mu(\infty,2)}\right)
\]
for $c_{0}$ sufficiently large. Similarly we can bound $\left\Vert \sqrt{\frac{n}{\mu_{a}r}}X_{a\cdot}\right\Vert _{2}$
by the same quantity. We take a union bound over all $a$ and $b$
to obtain the desired results.

\subsubsection{Proof of Lemma~\ref{lem:muinf}}

Fix a matrix index $(a,b)$ and let $w_{ab}=\sqrt{\frac{\mu_{a}r}{n}\frac{\nu_{b}r}{n}}$.
We can write
\begin{align*}
\left[\left(P_{T}R_{\Omega}-P_{T}\right)Z\right]_{ab}\sqrt{\frac{n}{\mu_{a}r}}\sqrt{\frac{n}{\nu_{b}r}} & =\sum_{i,j}\left(\frac{1}{p_{ij}}\delta_{ij}-1\right)Z_{ij}\left\langle e_{i}e_{j}^{\top},P_{T}(e_{a}e_{b}^{\top})\right\rangle \frac{1}{w_{ab}}=:\sum_{i,j}s_{ij},
\end{align*}
which is the sum of independent zero-mean variables. We first compute
the following bound:
\begin{align}
 & \left|\left\langle e_{i}e_{j}^{\top},P_{T}(e_{a}e_{b}^{\top})\right\rangle \right|\nonumber \\
= & \left|e_{i}^{\top}UU^{\top}e_{a}e_{b}^{\top}e_{j}+e_{i}^{\top}(I-UU^{\top})e_{a}e_{b}^{\top}VV^{\top}e_{j}\right|\nonumber \\
= & \begin{cases}
\left|e_{a}^{\top}UU^{\top}e_{a}+e_{a}^{\top}(I-UU^{\top})e_{a}e_{b}^{\top}VV^{\top}e_{b}\right|\le\frac{\mu_{a}r}{n}+\frac{\nu_{b}r}{n}, & i=a,j=b,\\
\left|e_{a}^{\top}(I-UU^{\top})e_{a}e_{b}^{\top}VV^{\top}e_{j}\right|\le\left|e_{b}^{\top}VV^{\top}e_{j}\right|, & i=a,j\neq b,\\
\left|e_{i}^{\top}UU^{\top}e_{a}e_{b}^{\top}(I-VV^{\top})e_{b}\right|\le\left|e_{i}^{\top}UU^{\top}e_{a}\right|, & i\neq a,j=b,\\
\left|e_{i}^{\top}UU^{\top}e_{a}e_{b}^{\top}VV^{\top}e_{j}\right|\le\left|e_{i}^{\top}UU^{\top}e_{a}\right|\left|e_{b}^{\top}VV^{\top}e_{j}\right|, & i\neq a,j\neq b,
\end{cases}\label{eq:bound4}
\end{align}
where we use the fact that the matrices $I-UU^{\top}$ and $I-VV^{\top}$
have spectral norm at most $1$. We proceed to bound $\left|s_{ij}\right|.$
Note that 
\[
\left|s_{ij}\right|\le\left(1-\mathbb{I}(p_{ij}=1)\right)\frac{1}{p_{ij}}\left|Z_{ij}\right|\left|\left\langle e_{i}e_{j}^{\top},P_{T}(e_{a}e_{b}^{\top})\right\rangle \right|\frac{1}{w_{ab}}.
\]
We distinguish four cases. When $i=a$ and $j=b$, we use~(\ref{eq:bound4})
and $p_{ab}\ge1\wedge\frac{c_{0}\left(\mu_{a}+\nu_{b}\right)r\log^{2}(n)}{n}$
to obtain $\left|s_{ij}\right|\le\left|Z_{ij}\right|/\left(w_{ij}c_{0}\log n\right)\le\left\Vert Z\right\Vert _{\mu(\infty)}/\left(c_{0}\log n\right).$
When $i=a$ and $j\neq b$, we apply~(\ref{eq:bound4}) to get 
\[
\left|s_{ij}\right|\le\left(1-\mathbb{I}(p_{ij}=1)\right)\frac{\left|Z_{aj}\right|}{p_{aj}}\cdot\sqrt{\frac{\nu_{b}r}{n}\frac{\nu_{j}r}{n}}\cdot\sqrt{\frac{n}{\mu_{a}r}\frac{n}{\nu_{b}r}}\overset{(a)}{\le}\left|Z_{aj}\right|\cdot\sqrt{\frac{n}{\mu_{a}r}\frac{n}{\nu_{j}r}}\frac{1}{c_{0}\log n}\le\frac{\left\Vert Z\right\Vert _{\mu(\infty)}}{c_{0}\log n},
\]
where $(a)$ follows from $p_{aj}\ge\min\left\{ c_{0}\frac{\nu_{j}r\log n}{n},1\right\} .$
In a similar fashion, we can show that the same bound holds when $i\neq a$
and $j=b$. When $i\neq a$ and $j\neq b$, we use~(\ref{eq:bound4})
to get
\begin{align*}
\left|s_{ij}\right|
&\le\left(1-\mathbb{I}(p_{ij}=1)\right)\frac{\left|Z_{ij}\right|}{p_{ij}}\cdot\sqrt{\frac{\mu_{i}r}{n}\frac{\mu_{a}r}{n}}\sqrt{\frac{\nu_{b}r}{n}\frac{\nu_{j}r}{n}}\cdot\sqrt{\frac{n}{\mu_{a}r}\frac{n}{\nu_{b}r}}\\
&\overset{(b)}{\le}\left|Z_{ij}\right|\cdot\sqrt{\frac{n}{\mu_{i}r}\frac{n}{\nu_{j}r}}\frac{1}{c_{0}\log n}
\le\frac{\left\Vert Z\right\Vert _{\mu(\infty)}}{c_{0}\log n},
\end{align*}
where $(b)$ follows from $p_{ij}\ge1\wedge c_{0}\sqrt{\frac{\mu_{i}r}{n}\frac{\nu_{j}r}{n}}\log n$
and $\max\left\{ \sqrt{\frac{\mu_{i}r}{n}},\sqrt{\frac{\nu_{j}r}{n}}\right\} \le1$.
We conclude that $\left|s_{ij}\right|\le\left\Vert Z\right\Vert _{\mu(\infty)}/\left(c_{0}\log n\right)$
for all $(i,j)$. 

On the other hand, note that 
\begin{align*}
\left|\mathbb{E}\left[\sum_{i,j}s_{ij}^{2}\right]\right|
&=\sum_{i,j}\mathbb{E}\left[\left(\frac{1}{p_{ij}}\delta_{ij}-1\right)^{2}\right]\frac{Z_{ij}^{2}}{w_{ab}^{2}}\left\langle e_{i}e_{j}^{\top},P_{T}(e_{a}e_{b}^{\top})\right\rangle ^{2}\\
&=\sum_{i=a,j=b}+\sum_{i=a,j\neq b}+\sum_{i\neq a,j=b}+\sum_{i\neq a,j\neq b}.
\end{align*}
We bound each of the four sums. By~(\ref{eq:bound4}) and $p_{ab}\ge1\wedge\frac{c_{0}(\mu_{a}+\nu_{b})r\log n}{n}\ge1\wedge\frac{c_{0}(\mu_{a}+\nu_{b})^{2}r^{2}\log n}{2n^{2}}$,
we have 
\[
\sum_{i=a,j=b}\le\frac{1-p_{ab}}{p_{ab}w_{ab}^{2}}Z_{ab}^{2}\left(\frac{\mu_{a}r}{n}+\frac{\nu_{b}r}{n}\right)^{2}\le\frac{2\left\Vert Z\right\Vert _{\mu(\infty)}^{2}}{c_{0}\log n}.
\]
By~(\ref{eq:bound4}) and $p_{aj}w_{ab}^{2}\ge w_{ab}^{2}\wedge\left(c_{0}w_{aj}^{2}\frac{\nu_{b}r}{n}\log n\right)$,
we have 
\[
\sum_{i=a,j\neq b}\le\sum_{,j\neq b}\frac{1-p_{aj}}{p_{aj}w_{ab}^{2}}Z_{aj}^{2}\left|e_{b}^{\top}VV^{\top}e_{j}\right|\le\frac{\left\Vert Z\right\Vert _{\mu(\infty)}^{2}}{c_{0}\log n}\cdot\frac{n}{\nu_{b}r}\sum_{j\neq b}\left|e_{b}^{\top}VV^{\top}e_{j}\right|,
\]
which implies $\sum_{i=a,j\neq b}\le\left\Vert Z\right\Vert _{\mu(\infty)}^{2}/(c_{0}\log n)$.
Similarly we can bound $\sum_{i\neq a,j=b}$ by the same quantity.
Finally, by~(\ref{eq:bound4}) and $p_{ij}\ge1\wedge\left(c_{0}\frac{\mu_{i}r}{n}\frac{\nu_{j}r}{n}\log n\right)$,
we have 
\begin{align*}
\sum_{i\neq a,j\neq b}
&\le\frac{1}{w_{ab}^{2}}\sum_{i\neq a,j\neq b}\frac{(1-p_{ij})Z_{ij}^{2}}{p_{ij}}\cdot\left|e_{i}^{\top}UU^{\top}e_{a}\right|\left|e_{b}^{\top}VV^{\top}e_{j}\right|\\
&\le\frac{\left\Vert Z\right\Vert _{\mu(\infty)}^{2}}{c_{0}\log n}\cdot\frac{1}{w_{ab}^{2}}\sum_{i\neq a}\left|e_{i}^{\top}UU^{\top}e_{a}\right|\sum_{j\neq b}\left|e_{b}^{\top}VV^{\top}e_{j}\right|,
\end{align*}
which implies $\sum_{i\neq a,j\neq b}\le\left\Vert Z\right\Vert _{\mu(\infty)}^{2}/(c_{0}\log n).$
Combining pieces, we obtain 
\[
\left|\mathbb{E}\left[{\textstyle \sum_{ij}}s_{ij}^{2}\right]\right|\le5\left\Vert Z\right\Vert _{\mu(\infty)}^{2}/(c_{0}\log n).
\]
Applying the Bernstein inequality (Theorem~\ref{lem:matrix_bernstein}),
we conclude that
\[
\left|\left[\left(P_{T}R_{\Omega}P_{T}-P_{T}\right)Z\right]_{ab}\sqrt{\frac{n}{\mu_{a}r}}\sqrt{\frac{n}{\nu_{b}r}}\right|=\left|\sum_{i,j}s_{ij}\right|\le\frac{1}{2}\left\Vert Z\right\Vert _{\mu(\infty)}
\]
w.h.p. for $c_{0}$ sufficiently large. The desired result follows
from a union bound over all $(a,b)$.

\section{Proof of Corollary~\ref{cor:row_coherent}}
\label{sec:proof_row_coherent}

Recall the setting: for each row of $ M $, we pick it and observe all its elements with some probability $ p $. We need a simple lemma. Let $J\subseteq[n]$ be the set of the indices of the row picked, and $ P_J(Z) $ be the matrix that is obtained from $ Z $ by zeroing out the rows outside $ J $. Recall that $ U\Sigma V^\top  $ is the SVD of $ M $. 
\begin{lemma}
If $\mu_i(M) = \max_{i} \frac{n}{r}\left\Vert U^{\top}e_{i}\right\Vert^2\le \mu_0$
and $p\ge c_0\frac{\mu_{0}r\log n}{n}$ for some universal constant $ c_0 $, then with probability at least $ 1-2n^{-10} $,
\[
\left\Vert U^{\top}P_{J}(U)-I_{r\times r}\right\Vert \le\frac{1}{2},
\]
where $I_{r\times r}$ is the identity matrix in $\mathbb{R}^{r\times r}$.\end{lemma}

\begin{proof}
Let $\eta_{j}=\mathbb{I}(i\in J)$, where $ \mathbb{I}(\cdot) $ is the indicator function. Note that
\[
U^{\top}P_{J}(U)-I_{r\times r}=U^{\top}P_{J}(U)-U^{\top}U=\sum_{i=1}^nS_{(i)}:=\sum_{i=1}^{n}\left(\frac{1}{p}\eta_{i}-1\right)U^{\top}e_{i}e_{i}^{\top}U.
\]
Note that $\mathbb{E}\left[S_{(i)}\right]=0$, 
$
\left\Vert S_{(i)}\right\Vert \le\frac{1}{p}\left\Vert U^{\top}e_{i}\right\Vert _{2}^{2}\le\frac{\mu_{0}r}{pn},
$
and 
\begin{align*}
\left\Vert\mathbb{E}\left[\sum_{i=1}^{n}S_{(i)}S_{(i)}^{\top}\right]\right\Vert 
=\left\Vert\mathbb{E}\left[\sum_{i=1}^{n}S_{(i)}^{\top}S_{(i)}\right]\right\Vert
 & =\frac{1-p}{p}\left\Vert \sum_{i=1}^{n}U^{\top}e_{i}e_{i}^{\top}UU^{\top}e_{i}e_{i}^{\top}U\right\Vert \\
 & =\frac{1-p}{p}\left\Vert U^{\top}\left(\sum_{i=1}^{n}e_{i}e_{i}^{\top}\left\Vert U^{\top}e_{i}\right\Vert _{2}^{2}\right)U\right\Vert \\
 & \le\frac{1}{p}\left\Vert \sum_{i=1}^{n}e_{i}e_{i}^{\top}\left\Vert U^{\top}e_{i}\right\Vert _{2}^{2}\right\Vert \\
 & =\frac{1}{p}\max_{i}\left\Vert U^{\top}e_{i}\right\Vert _{2}^{2}\le\frac{\mu_{0}r}{pn}.
\end{align*}
It follows from the matrix Bernstein (Theorem~\ref{lem:matrix_bernstein}) that with probability at least $ 1-2n^{-10} $
\[
\left\Vert U^{\top}P_{J}(U)-I_{r\times r}\right\Vert \le 20\max\left\{ \frac{\mu_{0}r}{pn}\log n,\sqrt{\frac{\mu_{0}r}{pn}\log n}\right\} \le\frac{1}{2}
\]
provided that the $ c_0 $ in the statement of the lemma is sufficiently large.
\end{proof}

Note that $\left\Vert U^{\top}P_{J}(U)-I_{r\times r}\right\Vert \le\frac{1}{2}$
implies that $U^{\top}P_{J}(U)$ is invertible, which further implies $P_{J}(U)\in\mathbb{R}^{n\times r}$
has rank-$r$. The rows picked are $P_{J}(M)=P_{J}(U)\Sigma V^{\top}$, which thus have full rank-$r$ and their row space must be the same as the row space of $M$. Therefore, the leverage scores $ \{\tilde{\nu}_j \}$ of these rows are the same as the row leverage scores $ \{\nu_j(M)\} $ of $ M $. Also note that we must have $ \mu_0\ge 1 $. Sampling $ \Omega $ as in described in the corollary and applying Theorem~\ref{thm:random}, we are guaranteed to recover $ M $ exactly with probability at least $ 1-5n^{-10} $. Note that expectation of  the total number of elements we have observed is 
$$
pn + \sum_{i,j}p_{ij} = c_0\mu_0 r \log n + c_0(\mu_0rn + rn)\log^2n \le 3c_0 \mu_0 r n \log^2 n,
$$
and by Hoeffding's inequality, the actual number of observations is at most two times the expectation with probability at least $ 1-n^{-10} $ provided $ c_0 $ is sufficiently large. The corollary follows from the union bound.

\section{Proof of Theorem \ref{thm:optimal} }

We prove the theorem assuming $ $$\sum_{k=1}^{r}\frac{1}{a_{k}}=\sum_{k=1}^{r}\frac{1}{b_{k}}=r$;
extension to the general setting in the theorem statement will only
affect the pre-constant in~(\ref{eq:too_small}) by a factor of at
most $2$. For each $k\in[r]$, let $s_{k}:=\frac{2n}{a_{k}r}$, $t_{k}:=\frac{2n}{b_{k}r}$.
We assume the $s_{k}$'s and $t_{k}$'s are all integers. Under the
assumption on $a_{k}$ and $b_{k}$, we have $1\le s_{k},t_{k}\le n$
and $\sum_{k=1}^{r}s_{k}=\sum_{k=1}^{r}t_{k}=n$. Define the sets
$I_{k}:=\left\{ \sum_{l=1}^{r-1}s_{l}+i:i\in[s_{k}]\right\} $ and
$J_{k}:=\left\{ \sum_{l=1}^{r-1}t_{l}+j:j\in[t_{k}]\right\} $; note
that $\bigcup_{k=1}^{r}I_{k}=\bigcup_{k=1}^{r}J_{k}=[n]$. The vectors
$\vec{\mu}$ and $\vec{\nu}$ are given by
\begin{align*}
\mu_{i} & =a_{k},\quad\forall k\in[r],i\in I_{k},\\
\nu_{j} & =b_{k},\quad\forall k\in[r],j\in J_{k}.
\end{align*}
It is clear that $\vec{\mu}$ and $\vec{\nu}$ satisfy the property
1 in the statement of the theorem.

Let the matrix $M^{(0)}$ be given by $M^{(0)}=AB^{\top}$, where
$A,B\in\mathbb{R}^{n\times r}$ are specified below. 
\begin{itemize}
\item For each $k\in[r]$, we set 
\[
A_{ik}=\sqrt{\frac{1}{s_{k}}}
\]
for all $i\in I_{k}$. All other elements of $A$ are set to zero.
Therefore, the $k$-th column of $A$ has $s_{k}$ non-zero elements
equal to $\sqrt{\frac{1}{s_{k}}}$, and the columns of $A$ have disjoint
supports. 
\item Similarly, for each $k\in[r]$ , we set
\[
B_{jk}=\sqrt{\frac{1}{t_{k}}}
\]
for all $j\in J_{k}$. All other elements of $B$ are set to zero. 
\end{itemize}
Observe that $A$ is an orthonormal matrix, so 
\[
\mu_{i}\left(M^{(0)}\right)=\frac{n}{r}\left\Vert A_{i\cdot}\right\Vert _{2}^{2}=\frac{n}{r}\cdot\frac{1}{s_{k}}=\frac{a_{k}}{2}=\frac{\mu_{i}}{2}\le\mu_{i},\forall k\in[r],i\in I_{k},.
\]
A similar argument shows that $\nu_{j}\left(M^{(0)}\right)\le\nu_{j},\forall j\in[n]$.
Hence $M^{(0)}\in\mathcal{M}_{r}\left(\vec{\mu},\vec{\nu}\right)$.
We note that $M^{(0)}$ is a block diagonal matrix with $r$
blocks where the $k$-th block has size $s_{k}\times t_{k}$, and
$\left\Vert M^{(0}\right\Vert _{F}=\sqrt{r}$.

Consider the $i_{0}$ and $j_{0}$ in the statement of the theorem.
There must exit some $ $$k_{1},k_{2}\in[r]$ such that $i_{0}\in I_{k_{1}}$
and $j_{0}\in J_{k_{2}}$. Assume w.l.o.g. that $s_{k_{1}}\ge t_{k_{2}}$.
then
\[
p_{i_{0}j_{0}}\le\frac{\mu_{i_{0}}+\nu_{j_{0}}}{4n}\cdot r\log\left(\frac{1}{\eta}\right)=\frac{a_{k_{1}}+b_{k_{2}}}{4n}\cdot r\log\left(\frac{1}{\eta}\right)=\frac{\log\left(1/\eta\right)}{4s_{k_{1}}}+\frac{\log\left(1/\eta\right)}{4t_{k_{2}}}\le\frac{\log\left(1/\eta\right)}{2t_{k_{2}}},
\]
where $\eta=\frac{\mu_{i_{0}}r}{2n}=\frac{1}{s_{k_{1}}}$ in part
2 of the theorem and $\eta=\frac{2}{n}$ is part 3. Because $\left\{ p_{ij}\right\} $
is location-invariant w.r.t. $M^{(0)}$, we have 
\[
p_{ij}=p_{i_{0}j_{0}}\le\frac{\log\left(1/\eta\right)}{2t_{k_{2}}},\quad\forall i\in I_{k_{1}},j\in J_{k_{2}}.
\]

Let $W_{i}:=\left|\left(\{i\}\times J_{k_{2}}\right)\cap\Omega\right|$
be the number of observed elements on $\{i\}\times J_{k_{2}}$. Note
that for each $i\in I_{k_{1}},$ we have 
\[
\mathbb{P}\left[W_{i}=0\right]=\prod_{j\in J_{k_{2}}}\left(1-p_{ij}\right)\ge\left(1-\frac{\log(1/\eta)}{2t_{k_{2}}}\right)^{t_{k_{2}}}\ge\exp\left(\log\eta\right)=\eta,
\]
where we use $1-x\ge e^{-2x},\forall0\le x\le\frac{1}{2}$ in the
second inequality. Therefore, there exists $i^{*}\in I_{k_{1}}$ for which there is no observed
element in $\left\{ i^{*}\right\} \times J_{k_{2}}$ with probability
\begin{align*}
\mathbb{P}\left[W_{i^{*}}=0,\exists i^{*}\in I_{k_{1}}\right]
&=1-\mathbb{P}\left[W_{i}\ge1,\forall i\in I_{k_{1}}\right]\\
&\ge 1-\left(1-\eta\right)^{s_{k_{1}}}
\ge1-e^{-\eta s_{k_{1}}}
\ge\frac{1}{2}\eta s_{k_{1}}
\ge
\begin{cases}
\frac{1}{2}, & \eta=\frac{\mu_{i_{0}}r}{4n}\\
\frac{1}{n}, & \eta=\frac{n}{2}.
\end{cases}
\end{align*}

Choose a number $ \bar{s} \ge s_{k_1} $. Let $M^{(1)}=\bar{A}B^{\top}$, where $B$ is the same as before and
$\bar{A}$ is given by
\[
\bar{A}_{ik}=\begin{cases}
-\sqrt{\frac{1}{s_{k_{1}}}}, & i=i^{*},k=k_{2}\\
A_{ik}, & \text{otherwise}.
\end{cases}
\]
By varying $\bar{s}$  we can construct infinitely many such $M^{(1)}$. Clearly $M^{(1)}$ is rank-$r$. Observe that $M^{(1)}$ differs from $M^{(0)}$ only in
$\left\{ i^{*}\right\} \times J_{k_{2}}$, which are not observed,
so 
\[
M_{ij}^{(0)}=M_{ij}^{(1)},\quad\forall(i,j)\in\Omega.
\]
Moreover, the number of elements that $M^{(0)}$ and $M^{(1)}$ differ
in is $\left|J_{k_{2}}\right|=t_{k_{2}}=2n/\left(b_{k_{2}}r\right)=2n/\left(\nu_{j_{0}}r\right),$
and 
\[
\frac{\left\Vert M^{(1)}-M^{(0)}\right\Vert _{F}^{2}}{\left\Vert M^{(0)}\right\Vert _{F}^{2}}=\frac{\left\Vert \left(\bar{A}-A\right)B^{\top}\right\Vert _{F}^{2}}{r}=\frac{t_{k_{2}}\cdot\frac{1}{s_{k_{1}}t_{k_{2}}}}{r}=\frac{\mu_{i_{0}}}{2n}.
\]
It is also easy to check that any $\left\{ p_{ij}\right\} $
location-invariant w.r.t. $M^{(0)}$ is also location-invariant to
$M^{(1)}$. The following lemma guarantees that $M^{(1)}\in\mathcal{M}_{r}\left(\vec{\mu},\vec{\nu}\right)$,
which completes the proof of the theorem.
\begin{lemma}
The matrix $ M^{(1)} $ constructed above satisfies
\begin{align*}
\mu_{i}\left(M^{(1)}\right) & \le2\mu_{i}\left(M^{(0)}\right), \quad\forall i\in[n],\\
\nu_{j}\left(M^{(1)}\right) & =\nu_{j}\left(M^{(0)}\right), \quad\forall j\in[n].
\end{align*}
\end{lemma}
\begin{proof}
Note that by the definition, the leverage scores of a rank-$ r $ matrix $M$
with SVD $M=U\Sigma V^{\top}$ can be expressed as
\[
\mu_{i}\left(M\right)=\frac{n}{r}\left\Vert U^{\top}e_{i}\right\Vert _{2}^{2}=\frac{n}{r}\left\Vert UU^{\top}e_{i}\right\Vert _{2}^{2}=\frac{n}{r}\left\Vert \mathcal{P}_{\text{col}(M)}(e_{i})\right\Vert _{2}^{2},
\]
where $\text{col}(M)$ denotes the column space of $M$ and $\mathcal{P}_{\text{col}(M)}(\cdot)$
is the Euclidean projection onto the column space of $M$. A similar relation
holds for the row leverage scores and the row space of $M$. In
other words, the column/row leverage scores of a matrix are determined
by its column/row space. Because $M^{(0)}$ and $M^{(1)}$ have the
same row space (which is the span of the columns of $B$), the second
set of equalities in the lemma hold.

It remains to prove the first set of inequalities for the  column leverage scores.
If $k_{1}=k_{2}$, then the columns of $\bar{A}$ have unit norms
and are orthogonal to each other. Using the above expression for the
leverage scores, we have
\[
\mu_{i}\left(M^{(1)}\right)=\frac{n}{r}\left\Vert \bar{A}\bar{A}^{\top}e_{i}\right\Vert _{2}^{2}=\frac{n}{r}\left\Vert \bar{A}^{\top}e_{i}\right\Vert _{2}^{2}=\frac{n}{r}\left\Vert A^{\top}e_{i}\right\Vert _{2}^{2}=\mu_{i}\left(M^{(0)}\right).
\]
If $k_{1}\neq k_{2}$, we may assume WLOG that $k_{1}=1$, $k_{2}=2$
and $i^{*}=1$. In the sequel we use $\bar{A}_{i}$ to denote the
$i$-th columns of $\bar{A}$. We now construct two vectors $\tilde{\alpha}$
and $\tilde{\beta}$ which have the same span with $\bar{A}_{1}$
and $\bar{A}_{2}$. Define two vectors $\alpha,\beta\in\mathbb{R}^{n}$,
such that the first $s_{1}$ elements of $\alpha$ and the $\left\{ s_{1}+1,\ldots,s_{1}+s_{2}\right\} $-th
elements of $\beta$ are one, the first element of $\beta$ is $\sqrt{\frac{s_{2}}{\bar{s}}}$,
and all other elements of $\alpha$ and $\beta$ are zero. Clearly
$\alpha=\sqrt{s_{1}}\bar{A}_{1}$ and $\beta=\sqrt{s_{2}}\bar{A}_{2}$,
so $\text{span}(\alpha,\beta)=\text{span}(\bar{A}_{1},\bar{A}_{2})$.
We next orthogonalize $\alpha$ and $\beta$ by letting $\bar{\alpha}=\alpha$
and 
\[
\bar{\beta}=\beta-\frac{\left\langle \alpha,\beta\right\rangle }{\left\Vert \alpha\right\Vert ^{2}}\alpha
=\beta-\frac{\sqrt{s_{2}}}{s_{1}\sqrt{s_{1}}}\alpha
=\begin{cases}
\frac{(s_{1}-1)\sqrt{s_{2}}}{s_{1}\sqrt{\bar{s}}}, & i=1\\
-\frac{\sqrt{s_{2}}}{s_{1}\sqrt{\bar{s}}}, & i=2,\ldots,s_{1}\\
1, & i=s_{1}+1,\ldots,s_{1}+s_{2}.
\end{cases}
\]
Note that $\text{span}(\bar{\alpha},\bar{\beta})=\text{span}(\alpha,\beta)$
and $\left\langle \bar{\alpha},\bar{\beta}\right\rangle =0$. Simple
calculation shows that $\left\Vert \bar{\alpha}\right\Vert _{2}^{2}=\left\Vert \alpha\right\Vert _{2}^{2}=s_{1}$
and $\left\Vert \bar{\beta}\right\Vert _{2}^{2}=\left(\frac{s_{1}-1}{s_{1}\bar{s}}+1\right)s_{2}.$
Finally, we normalize $\bar{\alpha}$ and $\bar{\beta}$ by letting
$\tilde{\alpha}=\bar{\alpha}/\left\Vert \bar{\alpha}\right\Vert $
and $\tilde{\beta}=\bar{\beta}/\left\Vert \bar{\beta}\right\Vert $.
It is clear that $\text{span}(\tilde{\alpha},\tilde{\beta})=\text{span}(\bar{A}_{1},\bar{A}_{2})$,
and $\left\langle \tilde{\alpha},\bar{A}_{k}\right\rangle =\left\langle \tilde{\beta},\bar{A}_{k}\right\rangle =0,\forall k=3,\ldots,r$. 

Now consider the matrix $\tilde{A}\in\mathbb{R}^{n\times r}$ obtained
from $\bar{A}$ by replacing the first two columns of $\bar{A}$ with
$\tilde{\alpha}$ and $\tilde{\beta}$, respectively. Because $\text{col}(\tilde{A})=\text{col}(\bar{A})=\text{col}(M^{(1)})$,
we have
\[
\mu_{i}\left(M^{(1)}\right)=\frac{n}{r}\left\Vert \mathcal{P}_{\text{col}(\tilde{A})}\left(e_{i}\right)\right\Vert ^{2}.
\]
But the columns of $\tilde{A}$ have unit norms and are orthogonal
to each other. It follows that 
\[
\mu_{i}\left(M^{(1)}\right)=\frac{n}{r}\left\Vert \tilde{A}\tilde{A}^{\top}e_{i}\right\Vert ^{2}=\frac{n}{r}\left\Vert \tilde{A}^{\top}e_{i}\right\Vert ^{2}.
\]
For $s_{1}+s_{2}<i\le n$, since $ \bar{s}\ge s_1 $ we have $\left\Vert \tilde{A}^{\top}e_{i}\right\Vert ^{2}=\left\Vert \bar{A}^{\top}e_{i}\right\Vert ^{2}=\left\Vert A^{\top}e_{i}\right\Vert ^{2}$
so $\mu_{i}\left(M^{(1)}\right)=\mu_{i}\left(M^{(0)}\right)$. For
$i\in[s_{1}+s_{2}]$, we have
\[
\left\Vert \tilde{A}^{\top}e_{i}\right\Vert ^{2}=\tilde{\alpha}_{i}^{2}+\tilde{\beta}_{i}^{2}
=\begin{cases}
\frac{1}{s_{1}}+\frac{(s_{1}-1)^{2}}{s_{1}(s_{1}-1)+s_{1}^{2}\bar{s}}\le\frac{2}{s_{1}}=2\left\Vert A^{\top}e_{i}\right\Vert ^{2}, & i=1\\
\frac{1}{s_{1}}+\frac{1}{s_{1}(s_{1}-1)+s_{1}^{2}\bar{s}}\le\frac{2}{s_{1}}=2\left\Vert A^{\top}e_{i}\right\Vert ^{2}, & i=2,\ldots,s_{1}\\
\frac{s_{1}\bar{s}}{(s_{1}-1+s_{1}\bar{s})s_{2}}\le\frac{1}{s_{2}}=\left\Vert A^{\top}e_{i}\right\Vert ^{2}, & i=s_{1}+1,\ldots,s_{1}+s_{2}.
\end{cases}
\]
This means 
\[
\mu_{i}\left(M^{(1)}\right)\le\frac{2n}{r}\left\Vert A^{\top}e_{i}\right\Vert ^{2}=2\mu_{i}(M^{(0)}),\forall i\in[s_{1}+s_{2}],
\]
which completes the proof of the lemma.
\end{proof}

\section{Proof of Theorem~\ref{cor:general_weighted}}\label{sec:proof_general_weighted}

Suppose the rank-$r$ SVD of $\bar{M}$ is $\bar{U}\bar{\Sigma}\bar{V}^{\top}$;
so $\bar{U}\bar{\Sigma}\bar{V}^{\top}=RMC=RU\Sigma V^{\top}C$. By
definition, we have
\[
\frac{\bar{\mu}_{i}r}{n}=\left\Vert P_{\tilde{U}}(e_{i})\right\Vert _{2}^{2},
\]
where $P_{\tilde{U}}(\cdot)$ denotes the projection onto the column
space of $\tilde{U}$, which is the same as the column space of $RU$.
This projection has the explicit form
\[
P_{\tilde{U}}(e_{i})=RU\left(U^{\top}R^{2}U\right)^{-1}U^{\top}Re_{i}.
\]
It follows that
\begin{align}
\frac{\bar{\mu}_{i}r}{n} & =\left\Vert RU\left(U^{\top}R^{2}U\right)^{-1}U^{\top}Re_{i}\right\Vert _{2}^{2}\nonumber \\
 & =R_{i}^{2}e_{i}^{\top}U\left(U^{\top}R^{2}U\right)^{-1}U^{\top}e_{i}\nonumber \\
 & \le R_{i}^{2}\left[\sigma_{r}\left(RU\right)\right]^{-2}\left\Vert U^{\top}e_{i}\right\Vert _{2}^{2}\nonumber \\
 & \le R_{i}^{2}\frac{\mu_{0}r}{n}\left[\sigma_{r}\left(RU\right)\right]^{-2},\label{eq:w1}
\end{align}
where $\sigma_{r}(\cdot)$ denotes the $r$-th singular value and
the last inequality follows from the standard incoherence assumption $ \max_{i,j}\{\mu_i,\nu_j\}\le \mu_0 $.
We now bound $\sigma_{r}\left(RU\right)$. Since $RU$ has rank $r$,
we have
\[
\sigma_{r}^{2}\left(RU\right)=\min_{\left\Vert x\right\Vert =1}\left\Vert RUx\right\Vert _{2}^{2}=\min_{\left\Vert x\right\Vert =1}\sum_{i=1}^{n}R_{i}^{2}\left|e_{i}^{\top}Ux\right|^{2}.
\]
If we let $z_{i}:=\left|e_{i}^{\top}Ux\right|^{2}$ for each $i\in[n]$,
then $z_{i}$ satisfies
\[
\sum_{i=1}^{n}z_{i}=\left\Vert Ux\right\Vert _{2}^{2}=\left\Vert x\right\Vert _{2}^{2}=1
\]
and by the standard incoherence assumption, 
\[
z_{i}\le\left\Vert U^{\top}e_{i}\right\Vert _{2}^{2}\left\Vert x\right\Vert _{2}^{2}\le\frac{\mu_{0}r}{n}.
\]
Therefore, the value of the minimization above is lower-bounded by
\begin{equation}
\begin{aligned}\min_{z\in\mathbb{R}^{n}} & \;\sum_{i=1}^{n}R_{i}^{2}z_{i}\\
\textrm{s.t.} & \;\sum_{i=1}^{n}z_{i}=1,\quad0\le z_{i}\le\frac{\mu_{0}r}{n},\; i=1,\ldots,n.
\end{aligned}
\label{eq:LP}
\end{equation}
From the theory of linear programming, we know the minimum is achieved
at an extreme point $z^{*}$ of the feasible set. The extreme point
$z^{*}$ satisfies $z_{i}^{*}\ge0,\forall i$ and $n$ linear equalities
\begin{align*}
\sum_{i=1}^{n}z_{i}^{*} & =1,\\
z_{i}^{*} & =0,\quad\;\textrm{for }i\in I_{1},\\
z_{i}^{*} & =\frac{\mu_{0}r}{n},\;\textrm{for }i\in I_{2}
\end{align*}
for some index sets $I_{1}$ and $I_{2}$ such that $I_{1}\cap I_{2}=\phi$,$\left|I_{1}\right|+\left|I_{2}\right|=n-1$.
It is easy to see that we must have $\left|I_{2}\right|=\left\lfloor \frac{n}{\mu_{0}r}\right\rfloor $.
Since $R_{1}\le R_{2}\le\ldots\le R_{n}$, the minimizer $z^{*}$
has the form 
\begin{align*}
z_{i}^{\ast}=\begin{cases}
\frac{\mu_{0}r}{n},& i=1,\ldots,\left\lfloor \frac{n}{\mu_{0}r}\right\rfloor ,\\
1-\left\lfloor \frac{n}{\mu_{0}r}\right\rfloor \cdot\frac{\mu_{0}r}{n},& i=\left\lfloor \frac{n}{\mu_{0}r}\right\rfloor +1,\\
0,& i=\left\lfloor \frac{n}{\mu_{0}r}\right\rfloor +2,\ldots,n,
\end{cases}
\end{align*}
and the value of the minimization~(\ref{eq:LP}) is at least 
\[
\sum_{i=1}^{\left\lfloor n/(\mu_{0}r)\right\rfloor }R_{i}^{2}\frac{\mu_{0}r}{n}.
\]
This proves that $\sigma_{r}^{2}\left(RU\right)\ge\frac{\mu_{0}r}{n}\sum_{i=1}^{\left\lfloor n/(\mu_{0}r)\right\rfloor }R_{i}^{2}.$
Combining with~(\ref{eq:w1}), we obtain that
\[
\frac{\bar{\mu}_{i}r}{n}\le\frac{R_{i}^{2}}{\sum_{i'=1}^{\left\lfloor n/(\mu_{0}r)\right\rfloor }R_{i}^{2}},\quad\frac{\bar{\nu}_{j}r}{n}\le\frac{C_{j}^{2}}{\sum_{j'=1}^{\left\lfloor n/(\mu_{0}r)\right\rfloor }C_{j'}^{2}};
\]
the proof for $\bar{\nu}_{j}$ is similar. Applying Theorem~\ref{thm:random} to the equivalent problem~\eqref{eq:weighted2} with the above bounds
on $\bar{\mu}_{i}$ and $\bar{\nu}_{j}$ proves the theorem.

\section{Matrix Bernstein Inequality\label{sec:bernstein}}
\begin{theorem}
[\citealt{tropp2012user}]\label{lem:matrix_bernstein}Let $X_{1},\ldots,X_{N}\in\mathbb{R}^{n_{1}\times n_{2}}$
be independent zero mean random matrices. Suppose 
\begin{equation}
\max\left\{ \left\Vert \sum_{k=1}^{N}X_{k}X_{k}^{\top}\right\Vert ,\left\Vert \sum_{k=1}^{N}X_{k}^{\top}X_{k}\right\Vert \right\} \le\sigma^{2}
\end{equation}
 and $\left\Vert X_{k}\right\Vert \le B$ almost surely for all $k$.
Then for any $c>0$, we have 
\begin{equation}
\left\Vert \sum_{k=1}^{N}X_{k}\right\Vert \le2\sqrt{c\sigma^{2}\log(n_{1}+n_{2})}+cB\log(n_{1}+n_{2}).
\end{equation}
with probability at least $1-(n_{1}+n_{2})^{-(c-1)}.$
\end{theorem}

\vskip 0.2in

\end{document}